\documentclass[letterpaper, 10 pt, conference]{ieeeconf}  

\IEEEoverridecommandlockouts                              
\usepackage{verbatim}
\usepackage{graphics} 
\usepackage{amsmath,amssymb,stmaryrd,mathtools, amsfonts}
\usepackage{multirow}
\usepackage{hhline}
\usepackage{algorithm}
\usepackage{footmisc}
\usepackage[noend]{algpseudocode}
\usepackage{subcaption}
\usepackage{dsfont}
\usepackage{algorithmicx}
\usepackage{multirow}
\usepackage[utf8]{inputenc}
\usepackage{import}
\usepackage{graphicx, graphics}
\usepackage{caption}
\usepackage{subcaption}
\usepackage{array}
\usepackage{color}
\usepackage[table]{xcolor}
\usepackage{siunitx}
\usepackage[short]{optidef}
\usepackage{float}
\usepackage{cite}
\PassOptionsToPackage{hyphens}{url}\usepackage{hyperref}
\definecolor{myedit}{rgb}{1.0,0,0}
\definecolor{mytodo}{rgb}{0,1.0,1.0}
\definecolor{lightgray}{rgb}{0.8, 0.8, 0.8}

\newtheorem{prop}{Proposition}
\newtheorem{thm}{Theorem}
\newtheorem{rem}{Remark}

\title{\LARGE \bf 
Collision Avoidance for Dynamic Obstacles with Uncertain Predictions using Model Predictive Control
}

\author{Siddharth H. Nair, Eric H. Tseng, Francesco Borrelli
\thanks{
SHN and FB are with the Model Predictive Control Laboratory, UC Berkeley. HET is with Ford Research and Advanced Engineering. Emails: 
\{siddharth\_nair, fborrelli\}@berkeley.edu, htseng@ford.com}
}

\begin{document}

\maketitle
\thispagestyle{empty}
\pagestyle{empty}

\begin{abstract}
We propose a Model Predictive Control (MPC) for collision avoidance between an autonomous agent and dynamic obstacles with uncertain predictions. The collision avoidance constraints are imposed by enforcing positive distance between convex sets representing the agent and the obstacles, and tractably reformulating them using Lagrange duality. This approach allows for smooth collision avoidance constraints even for polytopes, which otherwise require mixed-integer or non-smooth constraints. We consider three widely used descriptions of the uncertain obstacle position: 1) Arbitrary distribution with polytopic support, 2) Gaussian distributions and 3) Arbitrary distribution with first two moments known. For each case  we obtain deterministic reformulations of the collision avoidance constraints. The proposed MPC formulation optimizes over feedback policies to reduce conservatism in satisfying the collision avoidance constraints. The proposed approach is validated using simulations of traffic intersections in CARLA.
\end{abstract}

\section{Introduction}
\label{sec:introduction}
\subsection*{Motivation}
Autonomous vehicle technologies have seen a surge in popularity over the last decade, with the potential to improve flow of traffic, safety and fuel efficiency \cite{nhtsa}. While existing technology is being gradually introduced into structured scenarios such as highway driving and low-speed parking, autonomous driving in urban settings is challenging in general, due to the uncertainty in surrounding agents' behaviours.  Significant research has been devoted to building predictive models for these behaviours, to provide nominal trajectories for these agents and also characterize the uncertainty in deviating from the nominal predictions \cite{multipath_2019, trajectron_2020}.

In this work, we use these uncertain predictions of the surrounding agents (denoted as obstacles) to design a planning framework for the controlled agent for collision avoidance. Our main focus is to design a planner that can solve the planning problem  1) efficiently (measured as time to compute the solution, critical for real-time deployment), and 2) reliably (measured as low rate  of infeasibility,  critical for  fewer interventions of backup planners).  We investigate this problem in the context of constrained optimal control and use Model Predictive Control (MPC), the state-of-the-art technique for real-time optimal control \cite{morari1999model, benblog}. 

\subsection*{Related work}
MPC is a popular technique for real-time collision avoidance for autonomous driving \cite{shen2021collision, brudigam2021stochastic, zhou2018joint, wang2020non} and robotics \cite{de2021scenario, castillo2020real, bujarbaruah2021learning, zhu2019chance, dixit2021risk}. A typical MPC algorithm computes control inputs by solving a finite horizon, constrained optimal control problem in a receding horizon fashion \cite{borrelli2017predictive}. The collision avoidance problem involves checking if there is a non-empty intersection between two general sets (corresponding to the geometries of the agent and obstacle). This is a non-convex problem and NP-hard in general \cite{canny1988complexity}, and application-specific simplifications are commonly used for the improving the tractability of the resulting optimal control problem. The most common simplifications involve convexifying one or both of the sets as a 1) point, 2) affine space, 3) sphere, 4) ellipsoid or 5) polytope. Combinations of 1)-4) are convenient for their simplicity and computational tractability, but tend to be conservative since the shape of an actual car is not well represented by such sets. Polytopes offer compact representations for obstacles in tight environments and are popular in autonomous driving applications \cite{cruise}, but using them in the collision avoidance problem results in non-smooth constraints, which require specialized solvers or mixed-integer reformulations. The authors of \cite{zhang2020optimization} consider the dual perspective of collision avoidance for static obstacles; checking for existence of a separating hyperplane between two sets. This perspective allows for smooth collision checking conditions for convex, cone-representable sets (including polytopes) by introducing additional dual variables.

We use this dual formulation in the MPC for avoiding collisions with dynamic obstacles, in the presence of prediction uncertainty in the pose of the agent and obstacles. The MPC finds an optimal sequence of parameterized policies, assuming feedback from the agent's and obstacles' states. Compared to optimizing over open-loop sequences, optimizing over feedback policies enhances feasibility of the MPC optimization problem due to the ability to react to different trajectory realizations of the agent and obstacles along the prediction horizon. For computational efficiency, we convexify the problem for finding optimal solutions or detecting infeasibility quickly. The work \cite{soloperto2019collision} is the closest to our approach, which proposes a nonlinear Robust MPC scheme that optimizes over policies and uses the dual collision avoidance formulation with safety guarantees (under assumptions of a safety region and a backup policy). However, their policies do not account for feedback from the obstacles, and their formulation is specific to uncertainty distributions with polytopic support.

\subsection*{Contributions}
\begin{itemize}
    \item MPC formulations for collision avoidance with dynamic obstacles, for three different descriptions of the prediction uncertainty:  1) Arbitrary distributions with polytopic support, 2) Gaussian distributions, and 3) Arbitrary distributions with first two moments known.
    \item Deterministic reformulations of collision avoidance between the agent and dynamic obstacles, with convex geometries and uncertain trajectory predictions in closed-loop with a parameterized feedback policy over the agent's and obstacles' states.
\end{itemize}
\subsection*{Notation}
The index set $\{k_1,k_1+1,\dots, k_2\}$ is denoted by $\mathcal{I}_{k_1}^{k_2}$. For a proper cone $\mathcal{K}$ and $x,y\in\mathcal{K}$, we have $x\succeq_{\mathcal{K}}y\Leftrightarrow x-y\in\mathcal{K}$. The dual cone of $\mathcal{K}$ is given by the convex set $\mathcal{K}^*=\{y| y^\top x\geq 0, \forall x\in\mathcal{K}\}$. $||\cdot||_p$ denotes the $p-$norm. $\otimes$ denotes the Kronecker product.

\section{Problem Formulation}
\label{sec:prblm_f}
\subsection{Dynamics and Geometry of Agent and Obstacles}
Consider a controlled agent described by a linear time-varying discrete-time model
\small\begin{align}\label{eq:ev_dyn}
    x_{t+1}&=A_tx_t+B_tu_t+E_tw_t\nonumber\\
    p_t&=Cx_t+c_t
\end{align}\normalsize
where $x_t\in\mathbb{R}^{n_x}, u_t\in\mathbb{R}^{n_u}, w_t\in\mathbb{R}^{n_x}$ are the state, input and process noise respectively and $A_t,B_t,E_t,C,c_t$ are the system matrices at time $t$. The vector $p_t\in\mathbb{R}^n$ describes the position of the autonomous agent in a global, Cartesian coordinate system. Given the rotation matrix $R_t$ describing the orientation of the autonomous agent (with respect to the global coordinate system) at time $t$, define the space occupied by the autonomous agent as the set
\small\begin{align}\label{eq:ev_shape}
    \mathbb{S}_t(x_t)=\{z\in\mathbb{R}^n\ |\exists y\in\mathbb{R}^n:Gy\preceq_{\mathcal{K}}g,  z=R_ty+p_t \}
\end{align}\normalsize
where $p_t=Cx_t+c_t$,  $G\in\mathbb{R}^{l\times n}, g\in\mathbb{R}^{l}$ and $\mathcal{K}\subset\mathbb{R}^l$ is a closed convex cone with non-empty interior. The set $\{y| Gy\preceq_{\mathcal{K}}g\}$ is non-empty, convex and compact, and describes the space occupied by the un-oriented agent at the origin, and can denote various shapes for an appropriate choice of $\mathcal{K}$ (e.g., a polytope when $\mathcal{K}$ is the positive orthant or an ellipsoid when $\mathcal{K}$ is the second-order cone). 

Now suppose that there are $M$ obstacles, each described by the affine time-varying discrete-time dynamics
\small\begin{align}\label{eq:tv_dyn}
    o^i_{t+1}&=T_t^io_{t}^i+q^i_t+F^i_{t}n^i_{t},\nonumber\\ 
    p^i_{t}&=Co^i_{t},~~~~~~~~~~~~~~~~~~~\forall i\in\mathcal{I}_1^M
\end{align}\normalsize
 where $o^i_t\in\mathbb{R}^{n_x}, p^i_t\in\mathbb{R}^n, n^i_t\in\mathbb{R}^{n_x}$ are the state, position and
 process noise respectively and $T^i_t,q^i_t,F^i_t,C$ are the system matrices of the $i^{th}$ obstacle at time $t$. Now suppose that the orientation $R^i_t$ of the $i^{th}$ obstacle at time $t$ is given, and define the space occupied by the obstacle as the set
 \small\begin{align}\label{eq:tv_shape}
    \mathbb{S}^i_t(o^i_t)=\{z\in\mathbb{R}^n\ |\exists y\in\mathbb{R}^n:G^iy\preceq_{\mathcal{K}}g^i, z=R^i_ty+p^i_t\}
 \end{align}\normalsize
 where $p^i_t=C^i_to^i_t$ and the non-empty, convex and compact set $\{y| G^iy\preceq_{\mathcal{K}}g^i\}$ describes the un-oriented shape of the obstacle at the origin. We also introduce the notation $o_t=[o^{1\top}_t\dots, o^{M\top}_t]^\top, n_t=[n^{1\top}_t\dots, n^{M\top}_t]^\top$ to denote the stacked obstacle state and process noise vectors at time $t$, and $T_t=\text{blkdiag}(T_t^1,\dots,T_t^M), F_t=\text{blkdiag}(F_t^1,\dots,F_t^M), q_t=[q_t^{1\top}\dots q_t^{M\top}]^\top$ to define the combined dynamics of the obstacles as $o_{t+1}=T_to_t+q_t+F_tn_t$.
 \subsection{Uncertainty Description}
The presence of process noises $w_t$ and $n_t$ in the dynamics of the controlled agent \eqref{eq:ev_dyn} and the obstacles \eqref{eq:tv_dyn} adds uncertainty in the prediction of their state trajectories. In this paper, we consider three different descriptions of the process noise distributions as follows.
\begin{itemize}
    \item \textbf{D1}: The joint process noise  $[w^\top_t, n^\top_t]^\top$ are i.i.d. $\forall t\geq0$ and are given by an unknown distribution with compact support, \small$$[w_t^\top\ n^\top_t]^\top\in\mathcal{D}=\{d\in\mathbb{R}^{n_x+Mn_x}|\ ||\Gamma d||_\infty\leq \gamma\}$$\normalsize
    for $\gamma>0$ and non-singular $\Gamma$.
    \item \textbf{D2}: The joint process noise $[w^\top_t, n^\top_t]^\top$ are i.i.d. $\forall t\geq0$ and are given by the Gaussian distribution \small$$[w_t^\top\ n^\top_t]^\top\sim\mathcal{N}(0, \Sigma).$$\normalsize
    \item \textbf{D3}: The joint process noise $[w^\top_t, n^\top_t]^\top$ are i.i.d $\forall t\geq0$ and are given by an unknown distribution with known mean and covariance,
    \small\begin{align*}
        \mathbb{E}\left([w_t^\top\ n^\top_t]^\top\right)&=0,\\
        \mathbb{E}\left(([w_t^\top\ n^\top_t]^\top)([w_t^\top\ n^\top_t])\right)&=\Sigma.
    \end{align*} \normalsize
\end{itemize}
 \subsection{Model Predictive Control Formulation}
 \label{sec:mpc_approach}
We aim to design a state-feedback control $u_t=\pi(x_t,o_t)$ for the controlled agent such that it avoids collisions with the obstacles
while respecting polytopic state-input constraints given by $\mathcal{XU}=\{(x,u)\ | F_j^xx+F_j^uu\leq f_j,~\forall j\in\mathcal{I}_1^J\}$, where $F^x_j\in\mathbb{R}^{1\times n_x},F^u_j\in\mathbb{R}^{1\times n_u}, f_j\in\mathbb{R}~\forall j\in\mathcal{I}_1^J$. We propose to compute the feedback control using MPC, by solving the following finite-horizon constrained optimal control problem.
\small\begin{subequations}\label{opt:MPC_skeleton}
\begin{align}
\mathbf{OPT}_t(\mathbf{D}\in&\{\mathbf{D1},\mathbf{D2},\mathbf{D3}\}):\nonumber\\
 \min_{\substack{\boldsymbol{\theta}_t}}&\quad J_t(\mathbf{x}_t,\mathbf{u}_t)\label{opt:obj}\\
 \text{s.t. }&\quad x_{k+1|t}=A_{k}x_{k|t}+B_{k} u_{k|t}+E_kw_{k|t},\label{opt:EV_dyn}\\
&\quad o_{k+1|t}=T_ko_{k|t}+q_k+F_kn_{k|t},\label{opt:TV_dyn}\\
&\quad (\mathbf{w}_t, \mathbf{n}_t)\text{ given by }\mathbf{D},\label{opt:uncertainty_desc}\\
&\quad(\mathbf{x}_t,\mathbf{u}_t,\mathbf{o}_t)\in\mathcal{C}(\mathbf{D}),\label{opt:constr}\\
&\quad\mathbf{u}_t=\Pi_{\boldsymbol{\theta}_t}(\mathbf{x}_t,\mathbf{o}_t),\label{opt:gen_pol_class}\\
&\quad x_{t|t}=x_t,\ o_{t|t}=o_t,\label{opt:init}\\
&\quad\forall k\in\mathcal{I}_t^{t+N-1}\nonumber
\end{align}
\end{subequations}\normalsize
where $\mathbf{x}_t=[x^\top_{t|t},\dots$, $x^\top_{t+N|t}]^\top$ (similar notiation for $\mathbf{o}_t$) and  $\mathbf{u}_t=[u^\top_{t|t},\dots, u^\top_{t+N-1|t}]^\top$ (similar notation for $\mathbf{w}_t,\mathbf{n}_t$). The feedback control is given by the optimal solution of \eqref{opt:MPC_skeleton} as \small
\begin{align}\label{eq:MPC}u_t=\pi_{\mathrm{MPC}}(x_t,o_t)=u^\star_{t|t}\end{align}\normalsize where the feedback over the agent's and obstacles' states enter as \eqref{opt:init}.  The objective \eqref{opt:obj} penalizes deviation of the agent's trajectory from a desired reference. The obstacle avoidance constraints, and state-input constraints along the prediction horizon are summarised as $\mathcal{C}(\mathbf{D})$ in \eqref{opt:constr}, and depend on the uncertainty description assumed in \eqref{opt:uncertainty_desc}.  In \eqref{opt:gen_pol_class}, the control inputs $\mathbf{u}_t$ along the prediction horizon are given by a parameterized policy class that depends on predictions of the agent's and obstacles' trajectories. We solve \eqref{opt:MPC_skeleton} for the MPC \eqref{eq:MPC} in batch form by explicitly substituting for the equality constraints \eqref{opt:EV_dyn}, \eqref{opt:TV_dyn} and optimize over the policy parameters $\boldsymbol{\theta}_t$.  
 

\section{Collision Avoidance for Dynamic Obstacles with Uncertain Predictions}
\label{sec:MPC}

In this section, we detail the proposed MPC formulation for avoiding collisions with uncertain and dynamically moving obstacles. Section~\ref{ssec:pol} describes our choice of policy parameterization for \eqref{opt:gen_pol_class}. In~\ref{ssec:ca_dual}, we derive a continuous reformulation of the collision avoidance problem $\mathbb{S}_k(x_{k|t})\cap \mathbb{S}_k^i(o^i_{k|t})=\emptyset$, keeping $x_{k|t}, o^i_{k|t}$ fixed. Then we introduce the prediction uncertainties in $x_{k|t},o_{k|t}$, and derive deterministic reformulations of the collision avoidance constraints and state-input constraints for each uncertainty description in~\ref{ssec:det_reform}. Section~\ref{ssec:cost} describes the MPC cost function, and the MPC design is consolidated in~\ref{ssec:cvxMPC}.
\subsection{Policy Parameterization}\label{ssec:pol}
In~\ref{opt:gen_pol_class}, we use parameterised feedback policies $\Pi_{\boldsymbol{\theta}_t}(\mathbf{x}_t,\mathbf{o}_t)$ for the  control actions $\mathbf{u}_t$ (as in \eqref{opt:gen_pol_class}) along the prediction horizon. Consider the following input policy for time $k$,
\small\begin{align}\label{eq:policy}
    &u_{k|t}=h_{k|t}+\sum_{l=t}^{k-1}M_{l,k|t}w_{l|t} +K_{k|t}(o_{k|t}-\bar{o}_{k|t})
\end{align}\normalsize
which uses state feedback for the obstacles' states but affine disturbance feedback for feedback over the agent's states (cf. \cite{goulart2006optimization} for equivalence of state and disturbance feedback, \cite{balci2021covariance} for a recent application). The nominal states $\bar{o}_{k|t}$ are obtained as $\bar{o}_{k+1|t}=T_k\bar{o}_{k|t}+q_k~\forall k\in\mathcal{I}_t^{t+N-1}$, with $\bar{o}_{t|t}=o_t$. 

In \eqref{mat:AB},\eqref{mat:ER} of Appendix \ref{app:matrices}, we define the matrices $\mathbf{A}_t,\mathbf{B}_t,\mathbf{E}_t$ to express the agent's trajectory as a function of $(x_t,\mathbf{u}_t,\mathbf{w}_t)$ as $\mathbf{x}_t=\mathbf{A}_tx_t+\mathbf{B}_t\mathbf{u}_t+\mathbf{E}_t\mathbf{w}_t$. Similarly, the matrices $\mathbf{T}_t,\mathbf{q}_t,\mathbf{F}_t$ given by \eqref{mat:Tq},\eqref{mat:ER}, give the obstacles' trajectory as a function of $(o_t,\mathbf{n}_t)$ as $\mathbf{o}_t=\mathbf{T}_to_t+\mathbf{q}_t+\mathbf{R}_t\mathbf{n}_t$.
Defining $\mathbf{h}_t,\mathbf{M}_t,\mathbf{K}_t$ given by \eqref{mat:hK},\eqref{mat:M},
the control policies along the prediction horizon are $\mathbf{u}_t=\Pi_{\boldsymbol{\theta}_t}(\mathbf{x}_t,\mathbf{o}_t)=\mathbf{h}_t+\mathbf{M}_t\mathbf{w}_t+\mathbf{K}_t\mathbf{F}_t\mathbf{n}_t$, parameterized by $\boldsymbol{\theta}_t=(\mathbf{h}_t,\mathbf{M}_t,\mathbf{K}_t)$.

Note that although $\mathbf{o}_t$ doesn't necessarily depend on $\mathbf{x}_t$ ($\because \mathbf{n}_t$ may be independent from $\mathbf{w}_t$), the policies $\Pi_{\boldsymbol{\theta}_t}(\mathbf{x}_t,\mathbf{o}_t)$ modify the distribution of $\mathbf{x}_t$ in response to $\mathbf{o}_t$. Solving \eqref{opt:MPC_skeleton} over open-loop sequences (i.e., $\Pi_{\boldsymbol{\theta}_t}(\mathbf{x}_t,\mathbf{o}_t)=\mathbf{h}_t$) can be conservative because the agent-obstacle trajectories $(\mathbf{x}_t,\mathbf{o}_t)$ from a single control sequence $\mathbf{u}_t=\mathbf{h}_t$ must satisfy all the constraints regardless of the realizations of $\mathbf{w}_t, \mathbf{n}_t$.
\subsection{Collision Avoidance Constraint Reformulation by Dualization}\label{ssec:ca_dual}
Given the states $x_{k|t}$, $o^i_{k|t}$ of the agent and $i^{th}$ obstacle respectively, the collision avoidance constraint is given by $\mathbb{S}_k(x_{k|t})\cap\mathbb{S}^i_k(o^i_{k|t})=\emptyset$. This can be equivalently expressed as $\text{dist}(\mathbb{S}_k(x_{k|t}), \mathbb{S}^i_k(o^i_{k|t}))> 0$\footnote[1]{
In practice, we replace $>0$ with $\geq d_{min}$ for some small $d_{min}>0$.} where $\text{dist}(\mathbb{S}_k(x_{k|t}), \mathbb{S}^i_k(o^i_{k|t}))$ is the solution of the convex optimization problem
\small\begin{align}\label{eq:dist}
    \text{dist}(\mathbb{S}_k(x_{k|t}), \mathbb{S}^i_k(o^i_{k|t}))&=\min_{\substack{z_1\in\mathbb{S}_k(x_{k|t}), z_2\in\mathbb{S}^i_k(o^i_{k|t})}}||z_1-z_2||_2\nonumber\\
    &=\min_{z_1,z_2}~~~||z_1-z_2||_2\nonumber\\
    &~~~~~\text{ s.t}~~ GR^\top_k(z_1-p_{k|t})\preceq_{\mathcal{K}}g,\nonumber\\
    &~~~~~~~~~~~G^iR^{i\top}_k(z_2-p^i_{k|t})\preceq_{\mathcal{K}}g^i.
\end{align}\normalsize

In the following proposition, we use the above formulation \eqref{eq:dist} and Lagrange duality to express the set intersection problem $\mathbb{S}_k(x_{k|t})\cap\mathbb{S}^i_k(o^i_{k|t})=\emptyset$ as a convex feasibility problem of finding a separating hyperplane.
\begin{prop}\label{prop:oa_r}
Given the state and orientation of the agent $x_{k|t}, R_{k}$, and state and orientation of the $i^{th}$ obstacle $o^i_{k|t}, R^i_{k}$ at the $k^{th}$ prediction time step, we have
\small\begin{align}\label{eq:dual_oa}
    &\text{dist}(\mathbb{S}_k(x_{k|t}), \mathbb{S}^i_k(o^i_{k|t}))>0\nonumber\\
    \Leftrightarrow& \exists\lambda^i_{k|t},\nu^i_{k|t}\in\mathcal{K}^*:\ -\lambda^{i\top}_{k|t}(GR_k^\top (p_{k|t}-p^i_{k|t})+g)-\nu^{i\top}_{k|t}g^i>0,\nonumber\\ &\Vert\lambda^{i\top}_{k|t}GR_k^\top\Vert_2\leq1, \lambda^{i\top}_{k|t}GR^\top_k=-\nu^{i\top}_{k|t}G^iR^{i\top}_{k|t}.
\end{align}\normalsize
\end{prop}
\begin{proof}
Appendix~\ref{app:prop1proof}
\end{proof}
The feasibility of \eqref{eq:dual_oa} gives a separating hyperplane with normal {\small{$\mu=-\lambda^{i\top}_{k|t}GR_k^\top$}} for the sets {\small{$\mathbb{S}_k(x_{k|t}), \mathbb{S}^i_k(o^i_{k|t})$}}, as follows. For any {\small{$z_1\in\mathbb{S}_k(x_{k|t})$}}, {\small{$z_2\in\mathbb{S}^i_k(o^i_{k|t})$}} and {\small{$\lambda^{i}_{k|t},\nu^i_{k|t}\in\mathcal{K}^*$}} feasible for \eqref{eq:dual_oa}, we have {\small{$\lambda^{i\top}_{k|t}(g-GR^\top_k(z_1-p_{k|t}))\geq0$}}, {\small{$\nu^{i\top}_{k|t}(g^i-G^iR^{i\top}_k(z_1-p^i_{k|t}))\geq0$}}. Adding these two inequalities  and substituting for {\small{$\mu=-\lambda^{i\top}_{k|t}GR_k^\top$}}, we thus get {\small{$\mu^\top z_1-\mu^\top z_2\geq -\lambda^{i\top}_{k|t}(GR_k^\top (p_{k|t}-p^i_{k|t})+g)-\nu^{i\top}_{k|t}g^i>0 $}}. 

Next, we reformulate \eqref{eq:dual_oa} to address the non-determinism arising from the uncertainty in positions $p_{k|t}=Cx_{k|t}+c_k$, $p^i_{k|t}=Co^i_{k|t}$ along the prediction horizon due to $\mathbf{w}_t, \mathbf{n}_t$.
\subsection{Deterministic Constraint Reformulation}\label{ssec:det_reform}
Deterministic reformulations for the collision avoidance constraints \eqref{eq:dual_oa} together with the state-input constraints $\mathcal{XU}=\{(x,u)\ | F_j^xx+F^u_ju\leq f_j~\forall j\in\mathcal{I}_1^J\}$ for the state predictions $\mathbf{x}_t,\mathbf{o}_t$ in closed-loop with \eqref{eq:policy}, is presented next for each uncertainty description: $\mathbf{D1}$, $\mathbf{D2}$ and $\mathbf{D3}$.  

We introduce the constant matrices $S^x_k, S^u_k, S^{o,i}_k$ such that $S^x_k\mathbf{x}_t=x_{k|t}$, $S^u_k\mathbf{u}_t=u_{k|t}$, and $S^{o,i}_k\mathbf{o}_t=o^i_{k|t}$. Let $\mathbf{P}$ be a permutation matrix such that $[\mathbf{w}^\top_t\ \mathbf{n}_t^\top]^\top=\mathbf{P}\mathbf{v}_t$ where \small$\mathbf{v}_t=[w^\top_{t|t}\ n^\top_{t|t}\dots w^\top_{t+N-1|t}\ n^\top_{t+N-1|t}]^\top$\normalsize . Also, define \small$\lambda_{k|t}=[\lambda^{1\top}_{k|t}\dots,\lambda^{M\top}_{k|t}]^\top$, $\boldsymbol{\lambda}_{t}=[\lambda^{\top}_{t+1|t}\dots,\lambda^{\top}_{t+N|t}]^\top$\normalsize\ (similarly for $\nu_{k|t}$ , $\boldsymbol{\nu}_{t}$). Given a sequence of noise realisations $(\mathbf{w}_t,\mathbf{n}_t)$, define the set of feasible agent-obstacles joint realizations $(\mathbf{x}_t,\mathbf{u}_t,\mathbf{o}_t)$ in the lifted-space $(\mathbf{x}_t,\mathbf{u}_t,\mathbf{o}_t,\boldsymbol{\lambda}_t,\boldsymbol{\nu}_t)$ as 
\small\begin{align}
 \mathcal{S}_t(\mathbf{w}_t,\mathbf{n}_t)=\left\{\begin{bmatrix}\mathbf{x}_t\\\mathbf{u}_t\\\mathbf{o}_t\\
    \boldsymbol{\lambda}_{t}\\\boldsymbol{\nu}_{t}\end{bmatrix}\middle\vert\small{\begin{aligned}&\{\text{Collision avoidance constraints}\}\\
     &\lambda^i_{k|t},\nu^i_{k|t}\in\mathcal{K}^*,\Vert\lambda^{i\top}_{k|t}GR_{k}^\top\Vert_2\leq 1,\\
     &\lambda^{i\top}_{k|t}GR_{k}^\top=- \nu^{i\top}_{k|t}G^iR^{i\top}_{k},\\
     &\lambda^{i\top}_{k|t}(GR_{k}^\top(C(S^x_{k}\mathbf{x}_t-S^{o,i}_k\mathbf{o}_t)+c_t)+g)\\
     &<-\nu^{i\top}_{k|t}g_i,~\forall k\in\mathcal{I}_{t+1}^{t+N}, \forall i\in\mathcal{I}_1^M\\
     &\{\text{State-input constraints}\}\\
     &F_j^xS^x_{k+1}\mathbf{x}_{t}+F_j^uS^u_k\mathbf{u}_{t}\leq f_j,\\~&\forall k\in\mathcal{I}_{t}^{t+N-1}, \forall i\in\mathcal{I}_1^J\\
     &\{\text{Agent \& obstacles' predictions}\}\\
     &\mathbf{x}_t=\mathbf{A}_tx_t+\mathbf{B}_t\mathbf{u}_t+\mathbf{E}_t\mathbf{w}_t,\\&\mathbf{o}_t=\mathbf{T}_to_t+\mathbf{q}_t+\mathbf{F}_t\mathbf{n}_t
    \end{aligned}}\right\}
\end{align}\normalsize
We now express the reformulations for the considered uncertainty descriptions using this set.
\subsubsection{Robust Formulation for Uncertainty Description \textbf{D1}}
 We seek to tighten the obstacle avoidance constraints, and state-input constraints to find $\mathbf{u}_t$ such that the tuple $(\mathbf{x}_t,\mathbf{u}_t,\mathbf{o}_t)$ satisfies the aforementioned constraints for all realisations of $[w^\top_{k|t}\ n^\top_{k|t}]^\top\in\mathcal{D}$, $\forall k\in\mathcal{I}_t^{t+N-1}$. We can write this formally as
\small\begin{align}\label{constr:robust}
    \mathcal{C}(\mathbf{D1})=\bigcap\limits_{\substack{\forall\mathbf{P}[\mathbf{w}^\top_t\ \mathbf{n}^\top_t]^\top\in\mathcal{D}^N}}\mathcal{S}_t(\mathbf{w}_t,\mathbf{n}_t)
 \end{align}\normalsize
 \vskip -0.005 true in
 where $\mathcal{D}^N=\{\mathbf{d}|\ \Vert\mathbf{\Gamma}\mathbf{d}\Vert_\infty\leq\gamma\}$, $\mathbf{\Gamma}=I_N\otimes\Gamma$.
 \subsubsection{Chance Constraint Formulation for Uncertainty Description \textbf{D2}}
For uncertainty description $\mathbf{D2}$, we have that $[w^\top_t, n^\top_t]^\top\sim\mathcal{N}(0,\Sigma)$, i.i.d. $\forall t\geq 0$. Since the uncertainties now have unbounded support, we adopt a chance constrained formulation, where for some $0<\epsilon<<1$, we find $\mathbf{u}_t$ such that the tuple $(\mathbf{x}_t,\mathbf{u}_t,\mathbf{o}_t, \boldsymbol{\lambda}_t,\boldsymbol{\nu}_t)$ satisfies the obstacle avoidance constraints \eqref{eq:dual_oa} and state-input constraints with probability greater than $1-\epsilon$, given that $[w^\top_{k|t}\ n^\top_{k|t}]^\top\sim\mathcal{N}(\mu,\Sigma)$, $\forall k\in\mathcal{I}_t^{t+N-1}$. Formally, we write this set as
\small\begin{align}\label{constr:stochastic}
     \mathcal{C}(\mathbf{D2})=\left\{\begin{bmatrix}\mathbf{x}_t\\\mathbf{u}_t\\\mathbf{o}_t\\
    \boldsymbol{\lambda}_{t}\\\boldsymbol{\nu}_{t}\end{bmatrix}\middle\vert\mathbb{P}\left(\begin{bmatrix}\mathbf{x}_t\\\mathbf{u}_t\\\mathbf{o}_t\\
    \boldsymbol{\lambda}_{t}\\\boldsymbol{\nu}_{t}\end{bmatrix}\in\mathcal{S}_t(\mathbf{w}_t,\mathbf{n}_t)\right)
     \geq1-\epsilon\right\}
 \end{align}\normalsize
 where the probability measure $\mathbb{P}(\cdot)$ is over $\mathbf{v}_t=\mathbf{P}^\top[\mathbf{w}^\top_t\ \mathbf{n}^\top_t]^\top$, and constructed as the product measure of $N$ i.i.d. Gaussian distributions $\mathcal{N}(0,\Sigma)$. 
  \subsubsection{Distributionally Robust Formulation for Uncertainty Description \textbf{D3}}
For uncertainty description $\mathbf{D3}$, we have that $[w^\top_t, n^\top_t]^\top$ are i.i.d. $\forall t\geq 0$ and have known mean and covariance, $\mathbb{E}([w^\top_t, n^\top_t]^\top)=0$, $\mathbb{E}([w^\top_t, n^\top_t]^\top[w^\top_t, n^\top_t])=\Sigma$. Denote the mean and covariance of the stacked random variables $\mathbf{v}_t$ as $\boldsymbol{0}=[0^\top\dots,0^\top]^\top$, $\mathbf{\Sigma}=I_N\otimes\Sigma$. Now define the \textit{ambiguity set}\cite{rahimian2019distributionally} as
\small$$\mathcal{P}=\{\text{Probability distributions with }\mathbb{E}(\mathbf{v}_t)=\boldsymbol{0},\mathbb{E}(\mathbf{v}_t\mathbf{v}_t^\top)=\mathbf{\Sigma} \}.$$\normalsize 
We adopt a distributionally robust, chance constrained formulation, where for some $0<\epsilon<<1$, we find $\mathbf{u}_t$ such that the tuple $(\mathbf{x}_t,\mathbf{u}_t,\mathbf{o}_t, \boldsymbol{\lambda}_t,\boldsymbol{\nu}_t)$ satisfies the obstacle avoidance constraints \eqref{eq:dual_oa} and state-input constraints with probability greater than $1-\epsilon$, for all probability distributions in $\mathcal{P}$. Formally, we write this set as
\small\begin{align}\label{constr:dist_robust}
     \mathcal{C}(\mathbf{U3})=\left\{\begin{bmatrix}\mathbf{x}_t\\\mathbf{u}_t\\\mathbf{o}_t\\
    \boldsymbol{\lambda}_{t}\\\boldsymbol{\nu}_{t}\end{bmatrix}\middle\vert\inf_{\substack{P\in\mathcal{P}\\\mathbf{v}_t\sim P}}\mathbb{P}\left(\begin{bmatrix}\mathbf{x}_t\\\mathbf{u}_t\\\mathbf{o}_t\\
    \boldsymbol{\lambda}_{t}\\\boldsymbol{\nu}_{t}\end{bmatrix}\in\mathcal{S}_t(\mathbf{w}_t,\mathbf{n}_t)\right)
     \geq1-\epsilon\right\}.
 \end{align}\normalsize
 The next theorem provides deterministic reformulations of the constraint sets presented above, and establishes the feasible set of \eqref{opt:MPC_skeleton} in terms of the policy parameters $\boldsymbol{\theta}_t=(\mathbf{h}_t,\mathbf{M}_t,\mathbf{K}_t)$ in \eqref{eq:policy} and Lagrange multipliers $\boldsymbol{\lambda}_t,\boldsymbol{\nu}_t$ in \eqref{eq:dual_oa}.
 \begin{thm}\label{thm:constr_det_r}
 For the agent \eqref{eq:ev_dyn} in closed-loop with policy \eqref{eq:policy} and obstacles modelled by \eqref{eq:tv_dyn}, define the following functions {\small{$\forall k\in\mathcal{I}^{t+N}_t, \forall i\in\mathcal{I}_1^M, \forall j\in\mathcal{I}_1^J$}} in the dual variables $(\boldsymbol{\lambda}_t, \boldsymbol{\nu}_t)$ and policy parameters $\boldsymbol{\theta}_t=(\mathbf{h}_t,\mathbf{M}_t,\mathbf{K}_t)$:
 \vskip -0.1 true in
 \small
 \begin{align}
 &Y^i_{k|t}(\boldsymbol{\theta}_t,\lambda^i_{k|t},\nu^i_{k|t})=-\lambda^{i\top}_{k|t}g-\nu^{i\top}_{k|t}g^i-\lambda^{i\top}_{k|t}GR_k^\top c_t\nonumber\\&~~~~-\lambda^{i\top}_{k|t}GR_k^\top C(S^x_k(\mathbf{A}_t x_t+\mathbf{B}\mathbf{h}_t)-S^{o,i}_k(\mathbf{T}_to_t+\mathbf{q}_t)),\label{eq:Yik}\\
 &Z^i_{k|t}(\boldsymbol{\theta}_t,\lambda^i_{k|t},\nu^i_{k|t})=\lambda^{i\top}_{k|t}GR_k^\top C\begin{bmatrix}(S^x_k(\mathbf{B}_t\mathbf{M}_t+\mathbf{E}_t))^\top\\\mathbf{F}_t^\top(S^x_k\mathbf{B}_t\mathbf{K}_t-S^{o,i}_k)^\top\end{bmatrix}^\top\label{eq:Zik}\\
 &\bar{Y}^j_{k|t}(\boldsymbol{\theta}_t)=f_j-F_j^x S^x_{k+1}(\mathbf{A}_tx_t+\mathbf{B}\mathbf{h}_t)-F_j^u S^u_k\mathbf{h}_t,\\
 &\bar{Z}^j_{k|t}(\boldsymbol{\theta}_t)=\begin{bmatrix}(F_j^x S^x_{k+1}(\mathbf{B}_t\mathbf{M}_t+\mathbf{E}_t)+F_j^u S^u_k\mathbf{M}_t)^\top\\ ((F_j^x S^x_{k+1}+F_j^u S^{u}_k)\mathbf{B}_t\mathbf{K}_t\mathbf{F}_t)^\top\end{bmatrix}^\top,
 \end{align}
 \normalsize
 where all other quantities in the above expressions are constants defined earlier in the paper.
 Then deterministic reformulations of the feasible sets of \eqref{opt:MPC_skeleton} are given as follows:
 \begin{enumerate}
     \item For uncertainty description $\mathbf{D1}$ with support parameterised by $(\Gamma, \gamma)$, the feasible set is given by
     \footnotesize 
     \begin{align}  \mathcal{F}_t(\mathbf{D1})=\left\{\left(\begin{aligned}\mathbf{h}_t\\\mathbf{M}_t\\\mathbf{K}_t\\\boldsymbol{\lambda}_t\\\boldsymbol{\nu}_t\end{aligned}\right)\middle\vert\begin{aligned}
     \forall &k\in\mathcal{I}_{t}^{t+N-1}, \forall i\in\mathcal{I}_1^M,\forall j\in\mathcal{I}_1^J:\\
     &\lambda^i_{k+1|t},\nu^i_{k+1|t}\in\mathcal{K}^*,\Vert\lambda^{i\top}_{k+1|t}GR_{k+1}^\top\Vert_2\leq 1,\\
     &\lambda^{i\top}_{k+1|t}GR_{k+1}^\top=- \nu^{i\top}_{k+1|t}G^iR^{i\top}_{k+1},\\
     &Y^i_{k+1|t}(\boldsymbol{\theta}_t,\lambda^i_{k+1|t},\nu^i_{k+1|t})>\\
     &\gamma\Vert Z^i_{k+1|t}(\boldsymbol{\theta}_t,\lambda^i_{k+1|t},\nu^i_{k+1|t})\mathbf{P}\mathbf{\Gamma}^{-1}\Vert_1,\\
     &\bar{Y}^j_{k|t}(\boldsymbol{\theta}_t)-\gamma\Vert \bar{Z}^j_{k|t}(\boldsymbol{\theta}_t)\mathbf{P}\mathbf{\Gamma}^{-1}\Vert_1\geq 0
     \end{aligned}\right\}\label{fs:rob}
     \end{align}
     \normalsize
     \item For uncertainty description $\mathbf{D2}$ given by $\mathcal{N}(0,\Sigma)$ and risk level $\epsilon$, the feasible set is inner-approximated by
     \footnotesize 
     \begin{align}
       \mathcal{F}_t(\mathbf{D2})=\left\{\left(\begin{aligned}\mathbf{h}_t\\\mathbf{M}_t\\\mathbf{K}_t\\\boldsymbol{\lambda}_t\\\boldsymbol{\nu}_t\end{aligned}\right)\middle\vert\begin{aligned}
     \forall &k\in\mathcal{I}_{t}^{t+N-1}, \forall i\in\mathcal{I}_1^M,\forall j\in\mathcal{I}_1^J:\\
     &\lambda^i_{k+1|t},\nu^i_{k+1|t}\in\mathcal{K}^*,\Vert\lambda^{i\top}_{k+1|t}GR_{k+1}^\top\Vert_2\leq 1,\\
     &\lambda^{i\top}_{k+1|t}GR_{k+1}^\top=- \nu^{i\top}_{k+1|t}G^iR^{i\top}_{k+1},\\
     &Y^i_{k+1|t}(\boldsymbol{\theta}_t,\lambda^i_{k+1|t},\nu^i_{k+1|t})>\\
     &\gamma_{ca}\Vert Z^i_{k+1|t}(\boldsymbol{\theta}_t,\lambda^i_{k+1|t},\nu^i_{k+1|t})\mathbf{P}\mathbf{\Sigma}^{\frac{1}{2}}\Vert_2,\\
     &\bar{Y}^j_{k|t}(\boldsymbol{\theta}_t)\geq\gamma_{xu}\Vert \bar{Z}^j_{k|t}(\boldsymbol{\theta}_t)\mathbf{P}\mathbf{\Sigma}^{\frac{1}{2}}\Vert_2
     \end{aligned}\right\}\label{fs:stoc}
     \end{align}
     \normalsize
     where $\gamma_{ca}=\Phi^{-1}(1-\frac{\epsilon}{2NM})$ and $\gamma_{xu}=\Phi^{-1}(1-\frac{\epsilon}{2NJ})$, $\Phi^{-1}(\cdot)$ being the percentile function of the standard normal distribution.
     \item For uncertainty description $\mathbf{D3}$ parameterised by the covariance $\Sigma$ and risk level $\epsilon$, the inner-approximation of the feasible set, $\mathcal{F}_t(\mathbf{D3})$, is the same as $\mathcal{F}_t(\mathbf{D2})$, except with $\gamma_{ca}=\sqrt{\frac{2NM-\epsilon}{\epsilon}}$ and $\gamma_{xu}=\sqrt{\frac{2NJ-\epsilon}{\epsilon}}$.
 \end{enumerate}
 \end{thm}
 \begin{proof} Appendix~\ref{app:thm1proof}
 \end{proof}
 The constraint sets characterised in Theorem~\ref{thm:constr_det_r} are non-convex because the terms $Y^i_{k|t}(\cdot), Z^i_{k|t}(\cdot)~\forall k\in\mathcal{I}_{t+1}^{t+N}, i\in\mathcal{I}_1^M$ given by \eqref{eq:Yik}, \eqref{eq:Zik} are bilinear in $\boldsymbol{\lambda}_t,\boldsymbol{\theta}_t$. In \ref{ssec:cvxMPC}, we propose a convex approximation of these constraint sets.
 \subsection{Cost function}\label{ssec:cost}
A convex, quadratic cost is used for \eqref{opt:obj} to penalise deviations of the agent's nominal (certainty-equivalent) state-input trajectories $\bar{\mathbf{x}}_t=\mathbf{A}_tx_t+\mathbf{B}_t\mathbf{h}_t$, $\bar{\mathbf{u}}_t=\mathbf{h}_t$, from a given reference trajectory $\mathbf{x}^{\text{ref}}_t=[x^{\text{ref}\top}_t,\dots,x^{\text{ref}\top}_{t+N}]^\top$, $\mathbf{u}^{\text{ref}}_t=[u^{\text{ref}\top}_t,\dots,u^{\text{ref}\top}_{t+N-1}]^\top$ and with $\mathbf{Q}, \mathbf{R}\succ 0$,
\small
\begin{align}\label{eq:SMPC_cost}
    J_t(\bar{\mathbf{x}}_t,\bar{\mathbf{u}}_t
    )=&(\mathbf{x}^{\text{ref}}_t-\bar{\mathbf {x}}_{t})^\top \mathbf{Q}(\mathbf{x}^{\text{ref}}_t-\bar{\mathbf{x}}_{t})+(\mathbf{u}^{\text{ref}}_t-\bar{\mathbf {u}}_{t})^\top \mathbf{R}(\mathbf{u}^{\text{ref}}_t-\bar{\mathbf{u}}_{t}).
\end{align}
\normalsize
 \subsection{Convexified MPC Formulation}\label{ssec:cvxMPC}
We linearize  $Y^i_{k|t}(\cdot), Z^i_{k|t}(\cdot)~\forall k\in\mathcal{I}_{t+1}^{t+N}, i\in\mathcal{I}_1^M$ from \eqref{eq:Yik}, \eqref{eq:Zik} for time $t$, about the previous solution $\boldsymbol{\theta}^*_{t-1}, \boldsymbol{\lambda}^*_{t-1}$ to get affine functions $\mathcal{L}Y^i_{k|t}(\cdot), \mathcal{L}Z^i_{k|t}(\cdot)$ given by
\small
\begin{align*}
    \mathcal{L}Y^i_{k|t}(\boldsymbol{\theta}_t,\lambda^i_{k|t},\nu^i_{k|t})=&Y^i_{k|t}(\boldsymbol{\theta}_t,\lambda^{i*}_{k-1|t-1},\nu^{i}_{k|t})\\+&Y^i_{k|t}(\boldsymbol{\theta}^*_{t-1},\lambda^{i}_{k|t}-\lambda^{i*}_{k-1|t-1},\nu^{i}_{k|t})\\
    \mathcal{L}Z^i_{k|t}(\boldsymbol{\theta}_t,\lambda^i_{k|t},\nu^i_{k|t})=&Z^i_{k|t}(\boldsymbol{\theta}_t,\lambda^{i*}_{k-1|t-1},\nu^{i}_{k|t})\\+&Z^i_{k|t}(\boldsymbol{\theta}^*_{t-1},\lambda^{i}_{k|t}-\lambda^{i*}_{k-1|t-1},\nu^{i}_{k|t})
\end{align*}
\normalsize
When $\gamma>0$, $\epsilon<\min\{NM,NJ\}$, the constraints
\small
\begin{align*}
    &\mathcal{L}Y^i_{k|t}(\boldsymbol{\theta}_t,\lambda^i_{k|t},\nu^i_{k|t})>\gamma\Vert \mathcal{L}Z^i_{k|t}(\boldsymbol{\theta}_t,\lambda^i_{k|t},\nu^i_{k|t})\mathbf{P}\mathbf{\Gamma}^{-1}\Vert_1,\\
    &\mathcal{L}Y^i_{k|t}(\boldsymbol{\theta}_t,\lambda^i_{k|t},\nu^i_{k|t})>\gamma_{ca}\Vert \mathcal{L}Z^i_{k|t}(\boldsymbol{\theta}_t,\lambda^i_{k|t},\nu^i_{k|t})\mathbf{P}\mathbf{\Sigma}^{\frac{1}{2}}\Vert_2
\end{align*}
\normalsize
are second-order cone (SOC) representable (LP representable in the first case)  because the composition of a SOC constraint with an affine map is still an SOC constraint.  Substituting these affine functions in the set definitions of {\small{$\mathcal{F}_t(\mathbf{D}_1),\mathcal{F}_t(\mathbf{D}_2),\mathcal{F}_t(\mathbf{D}_3)$}}, the resulting constraint sets {\small{$\Tilde{\mathcal{F}}_t(\mathbf{D}_1),\Tilde{\mathcal{F}}_t(\mathbf{D}_2),\Tilde{\mathcal{F}}_t(\mathbf{D}_3)$}} are convex. This follows from 1) convexity of the dual cone $\mathcal{K}^*$, 2) convexity of constraints {\small{$\Vert\lambda^{i\top}_{k|t}GR_k^\top\Vert_2\leq 1, \lambda^{i\top}_{k|t}GR_k^\top=- \nu^{i\top}_{k|t}G^iR^{i\top}_k$}}, and 3) convexity of constraints {\small{$\bar{Y}^j_{k|t}(\boldsymbol{\theta}_t)\geq\gamma\Vert \bar{Z}^j_k(\boldsymbol{\theta}_t)\mathbf{P}\mathbf{\Gamma}^{-1}\Vert_1$}}, {\small{$\bar{Y}^j_{k|t}(\boldsymbol{\theta}_t)\geq\gamma_{xu}\Vert \bar{Z}^j_k(\boldsymbol{\theta}_t)\mathbf{P}\mathbf{\Sigma}^{\frac{1}{2}}\Vert_2$}}  due to composition of SOC constraint with affine maps {\small{$\bar{Y}^j_{k|t}(\cdot), \bar{Z}^j_{k|t}(\cdot)$}}.
 The resulting optimization problem for our MPC for either of the uncertainty descriptions $\mathbf{D1}$, $\mathbf{D2}$ or $\mathbf{D3}$ is given by the convex optimization problem: 
\small
\begin{align}\label{opt:MPC_final}
\mathbf{OPT}^{CVX}_t(\mathbf{D}\in&\{\mathbf{D1},\mathbf{D2},\mathbf{D3}\}):\nonumber\\
    \min\limits_{\substack{\{\mathbf{h}_t,\mathbf{K}_t,\mathbf{M}_t,\boldsymbol{\lambda}_t,\boldsymbol{\nu}_t\}}}&~~J_t(\bar{\mathbf{x}}_t,\bar{\mathbf{u}}_t)\nonumber\\
    \text{s.t }&~ \bar{\mathbf{x}}_t=\mathbf{A}_tx_t+\mathbf{B}_t\mathbf{h}_t,~\bar{\mathbf{u}}_t=\mathbf{h}_t,\nonumber\\
    &~\{\mathbf{h}_t,\mathbf{K}_t,\mathbf{M}_t,\boldsymbol{\lambda}_t,\boldsymbol{\nu}_t\}\in\Tilde{\mathcal{F}}_t(\mathbf{D}).
\end{align}
\normalsize 
When the cone $\mathcal{K}$ is given by the positive orthant (for polytopic shapes) or the second-order cone (for ellipsoidal shapes), the optimization problem \eqref{opt:MPC_final} is given by a second-order cone program which can be efficiently solved. 
The optimal solution to \eqref{opt:MPC_final} is used to obtain the control action $u^*_{t|t}$ given by \eqref{eq:policy}.
\begin{rem}
The feasible set of \eqref{opt:MPC_final} is not a convex inner-approximation of the original problem with {\small{$\mathcal{F}_t(\mathbf{D})$}} from Theorem~\ref{thm:constr_det_r}. However, at the cost of introducing several new variables, a convex-inner approximation can be obtained by enforcing the collision avoidance constraints for all points in the convex relaxation of the bilinear equalities \eqref{eq:Yik}, \eqref{eq:Zik} given by McCormick envelopes\cite{mccormick1976computability}. An investigation along these lines is left for future research.
\end{rem}

\section{Simulations}
\label{sec:results}
In this section, we demonstrate our MPC formulation via two numerical examples\footnote[2]{Experiments were run on a computer with a Intel i9-9900K CPU, 32 GB RAM, and a RTX 2080 Ti GPU.} of a traffic intersection: 1) A longitudinal control example comparing the MPC formulations for each uncertainty description $\mathbf{D}\in\{\mathbf{D1},\mathbf{D2},\mathbf{D3}\}$, and 2) An unprotected left turn in CARLA, comparing the proposed approach against \cite{nair2021stochastic} to highlight the benefit of the proposed collision avoidance formulation.
\subsection{Longitudinal Control Example}\label{ssec:long}
\subsubsection{Models and Geometry}
We simulate an autonomous vehicle as the controlled agent and $M=2$ surrounding vehicles as obstacles at a traffic intersection as in Figure~\ref{fig:low_sim}. The vehicles' are modelled as $4.8m\times 2.8m$ rectangles, and their dynamics are given by Euler-discretized double integrator dynamics with $\Delta t=0.1s$,  states: [$s$ (longitudinal position, $v$ (speed)], control input: $a$ (acceleration) and keep the lateral coordinate constant. For obstacle predictions \eqref{eq:tv_dyn}, we use forecasts  of the acceleration inputs for each obstacle. 
\subsubsection{Process Noise Distribution}
We model $[w^\top_t\ n^\top_t]^\top$ as a product of 6 (2 for agent, 2$\times$2 for obstacles) independent uni-variate truncated normal random variables (cf. \cite{wang2020non,de2021scenario}),  $\text{truncNorm}(\mu,\sigma,a,b)$, with $\mu=0$, $a=-2$, $b=2$ and $\sigma=0.01$ for $w_t$, $\sigma=0.1$ for $n_t$.
The resulting distribution for $[w^\top_t\ n^\top_t]^\top$ has mean $0$, variance $\Sigma=7.7\cdot\text{blkdiag}(10^{-5}I_{2\times 2}, 10^{-3}I_{4\times 4})$ and support $\mathcal{D}=([-0.02, 0.02])^2\times([-0.2, 0.2])^4$ (with \small$\Gamma=\text{blkdiag}(10^{2}I_{2\times 2}, 10I_{4\times 4}), \gamma=2$\normalsize\ ), which is used for defining the uncertainty descriptions $\mathbf{D1}, \mathbf{D2}, \mathbf{D3}$.
\subsubsection{Constraints and Cost} We set horizon $N=12$ and choose cost matrices $\mathbf{Q}=10I_{2N}$,$ \mathbf{R}=20I_{N}$ to penalise deviations from set-point $[2s_{final}, 0]$, along with constraints on speed $v\in[0,12]ms^{-1}$ and acceleration $a\in[-6,5]ms^{-2}$. Every chance constraint is imposed with the same risk level, $\epsilon=0.0228$ to yield $\gamma_{ca}=\gamma_{xu}=2$ for $\mathbf{D2}$, and $\gamma_{ca}=\gamma_{xu}= 6.55$ for $\mathbf{D3}$. 
\subsubsection{Simulation Setup}
We compare the following control policies for the agent corresponding to the different uncertainty descriptions: 1) Robust MPC (RMPC) for $\mathbf{D1}$, 2) Stochastic MPC (SMPC) for $\mathbf{D2}$, and 3) Distributionally Robust MPC (DRMPC) for $\mathbf{D3}$. We run 10 simulations for each policy, starting from $x_0=[3m, 11.8ms^{-1}]$ until the agent reaches $s_{final}=50m$. If \eqref{opt:MPC_final} is infeasible, the brake $a=-6ms^{-2}$ is applied. In Figure~\ref{fig:low_sim}, the first obstacle has a PD controller to go south across the intersection at high speed, while the second obstacle has a PD controller to stop at the intersection.  The first obstacle is re-spawned after crossing the intersection by $20m$. Casadi \cite{andersson2019casadi} is used for modelling the problem \eqref{opt:MPC_final} with Gurobi \cite{gurobi} as the solver.
\subsubsection{Results}
The performance metrics for all the runs are recorded in  Table~\ref{tab:comparison_scenario_all}. We record $\%$ of time steps where constraint violations (in particular, the collision avoidance and speed constraints) and MPC infeasibility were detected, along with average times for solving \eqref{opt:MPC_final} and reaching $s_{final}$ and finally, the average values of the closest distance from the obstacles (computed using \eqref{eq:dual_oa}). In  Figure~\ref{fig:low_sim}, we compare the various formulations for a particular run. 
\begin{figure}[!ht]
\vskip -0.1 true in
  \centering
  \includegraphics[width=0.9\columnwidth]{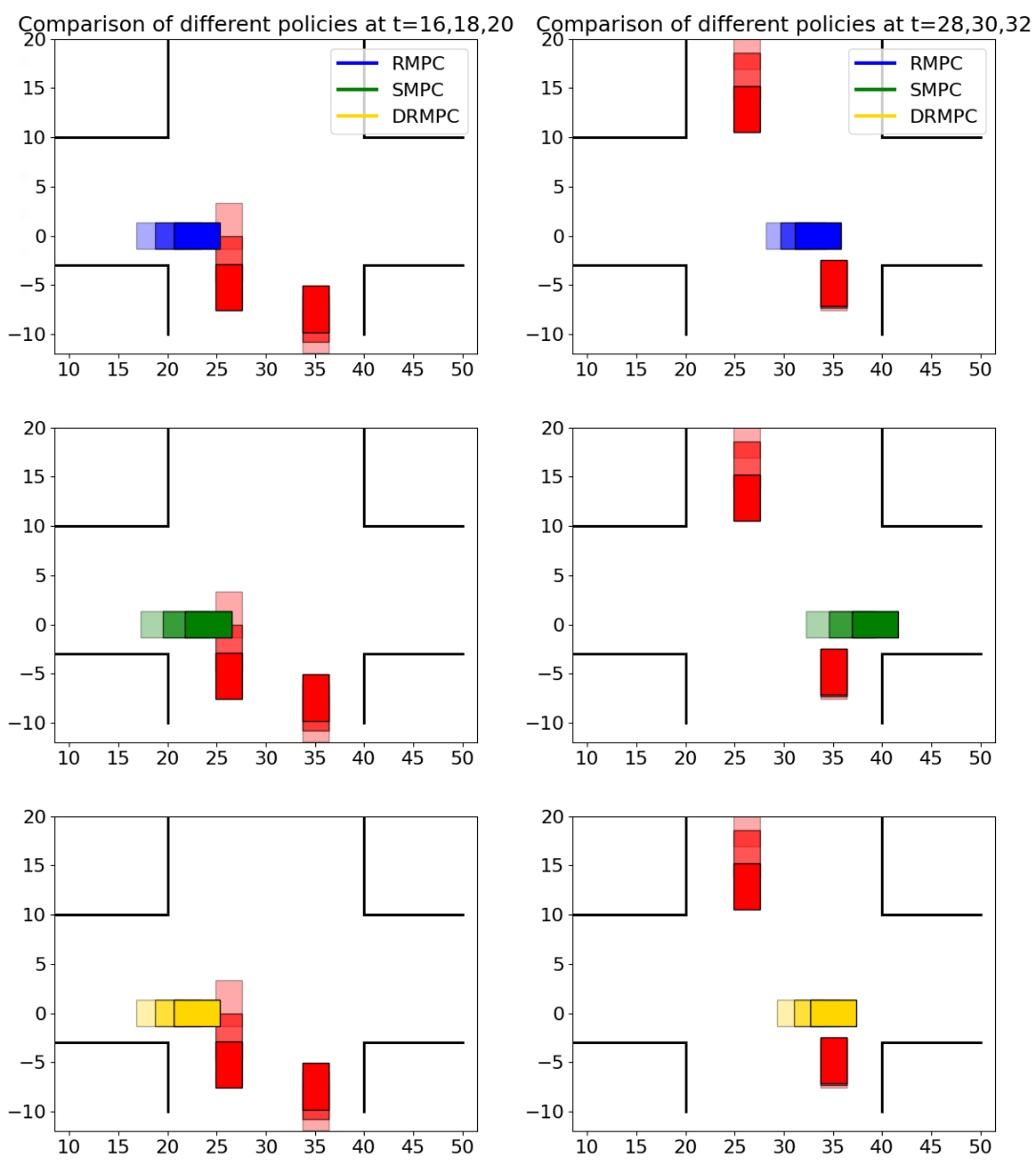}
\vskip -0.1 true in
\caption{\small{Snapshots of agents with RMPC, SMPC and DRMPC for a particular run. Darker colors correspond to later time steps. On the left, all agents slow down and cross behind the first obstacle. On the right, the agents are speeding up to cross the intersection, after slowing down in response to the uncertain second obstacle at the stop. Video: \url{https://youtu.be/wgqO36a1SU8}}}
\label{fig:low_sim}
\vskip -0.28 true in
\end{figure}
\begin{table}[!ht]
\vskip -0.15 true in
    \centering
    \caption{\small{Performance metrics across all policies for \ref{ssec:long}.}}
    \vskip -0.1 true in
    \label{tab:comparison_scenario_all}    
    \begin{tabular}[c]{|c || c | c | c |}
        \hline
            Performance metric &
        RMPC   & SMPC &
        DRMPC\\
        \hline
        \hline
        Constraint violations (\%)    &$\mathbf{1.07}$ & 3.88 & 3.69\\
       Feasibility (\%)    & 95.24 & $\mathbf{97.46}$ & 96.58\\
       
         Avg. solve time (ms)    & $\mathbf{32.62}$ & 54.10 & 54.33\\
               Avg. task completion time (s)    & 9.15 & $\mathbf{8.59}$ & 8.73 \\
               Avg. min. distance from obstacles (m) & 0.52 & 0.19 & 0.36\\
        \hline
    \end{tabular}
    \vskip -0.15 true in
\end{table}
\subsubsection{Discussion} In Figure~\ref{fig:low_sim} and Table~\ref{tab:comparison_scenario_all}, we see that in terms of conservatism (feasibility, time to reach $s_{final}$, constraint violations) the polices are ordered as: SMPC$>$DRMPC$>$RMPC, with SMPC being least conservative. Using the equivalence of norms and {\small{$\mathbf{\Sigma}^{\frac{1}{2}}=0.88\mathbf{\Gamma}^{-1}$}} for our example, it can be seen that {\small{$\tilde{\mathcal{F}}_t(\mathbf{D1})\subset\tilde{\mathcal{F}}_t(\mathbf{D2})$}}, {\small{$\tilde{\mathcal{F}}_t(\mathbf{D3})\subset\tilde{\mathcal{F}}_t(\mathbf{D2})$}}. The relation between {\small{$\tilde{\mathcal{F}}_t(\mathbf{D1}),\tilde{\mathcal{F}}_t(\mathbf{D3})$}} can't be established this way and needs further study. The increased conservatism for RMPC, however, results in fewer constraint violations and its LP formulation of collision avoidance constraints yields faster solve times.
\subsection{Unprotected Left Turn in CARLA}\label{ssec:carla}
We use our setup from \cite{nair2021stochastic} for the next couple of experiments, where we simulate an autonomous vehicle as the controlled agent and $M=1$ other vehicle as the obstacle at a traffic intersection in CARLA\cite{carla_sim_2017}. The agent is tasked to turn left while avoiding collision with the oncoming obstacle.
\subsubsection{Models and Geometry}
The agent's dynamics \eqref{eq:ev_dyn} are modelled by the kinematic bicycle model linearized about the reference, and the obstacle's predictions \eqref{eq:tv_dyn} are given by our implementation of Multipath\cite{multipath_2019} for $N=10$ steps. We use uni-modal predictions for Experiment 1 and multi-modal predictions with 3 modes for Experiment 2. The vehicles' geometries are given by $4.9m\times2.8m$ rectangles. 
\subsubsection{Process Noise Distribution}$[w^\top_t\ n^\top_t]^\top$ is given by the Gaussian distribution used in \cite{nair2021stochastic}. 
\subsubsection{Constraints and Cost} We set $N=10$ and choose $\mathbf{Q}=\text{blkdiag}(Q_1,10,1,\dots,Q_N,10,1)$, where $Q_t=5R^\top_t\text{diag}(1,10^{-2})R_t$, and $\mathbf{R}=I_N\otimes\text{diag}(10,10^3)$ for the cost to penalize deviations from the reference. We impose constraints on speed $v\in[0,12]ms^{-1}$,  acceleration $a\in[-3,2]ms^{-2}$ and steering $\delta\in[-0.5,0.5]$. Every individual chance constraint is imposed with $\epsilon=0.05$ to yield $\gamma_{ca}=\gamma_{xu}=1.64$ for $\mathbf{D2}$, and $\gamma_{ca}=\gamma_{xu}= 4.36$ for $\mathbf{D3}$. For the multi-modal predictions in Experiment 2, the constraints \eqref{eq:dual_oa} are imposed for each mode of the obstacle, and the SMPC finds a single policy sequence \eqref{eq:policy} that satisfies the tightened constraints for all modes.  
\subsubsection*{Experiment 1: SMPC vs DRMPC}
Since the prediction uncertainties are unbounded, we only compare SMPC for $\mathbf{D2}$, and  DRMPC for $\mathbf{D3}$. We run 10 simulations for each policy with different initial conditions.  If \eqref{opt:MPC_final} is infeasible, the brake $a=-6ms^{-2}$ is applied. 
The results are tabulated in Table~\ref{tab:comparison_CARLA}, where we see that compared to SMPC, the agent with DRMPC stays further away from the obstacle and deviates more from the reference. However, this enables maintaining a higher speed and completing the task faster.
\begin{table}[ht]
\vskip -0.1 true in
    \centering
    \caption{\small{Performance metrics for \ref{ssec:carla}, Experiment 1}}
    \vskip -0.1 true in
    \label{tab:comparison_CARLA}    
    \begin{tabular}[c]{|c || c | c |}
        \hline
            Performance metric  & SMPC &
        DRMPC\\
        \hline
        \hline
       Feasibility (\%)    & 97.3 & $\mathbf{98.3}$\\
       
         Avg. solve time (ms)    &  $\mathbf{24.4}$ & 27.1\\
               Avg. task completion time (s)    &  11.17 & $\mathbf{10.32}$ \\
        Avg. min. distance from obstacle (m) & 3.36 & 3.74\\
        Avg. Hausdorff distance from reference (m) & $\mathbf{0.94}$ & 1.24\\
        \hline
    \end{tabular}
    \vskip -0.1 true in
\end{table}
\subsubsection*{Experiment 2: SMPC vs SMPC of \cite{nair2021stochastic}}
We compare the SMPC approach in this paper with that of \cite{nair2021stochastic}, where collision is modelled as the intersection of an ellipse (for the obstacle) and a circle (for the agent), and the free space is inner-approximated using an affine constraint. While the latter is robust to deviations of the agent's orientation along the predictions, this approach induces conservative and undesirable maneuvers for collision avoidance. We summarise our findings in Figure~\ref{fig:carla_sim}, and observe that the new approach allows for a tighter left-turn in Figure~\ref{fig:hfsim2}.
\begin{figure}
    \begin{subfigure}{1.05\columnwidth}
  \centering
  \includegraphics[width=1.0\columnwidth]{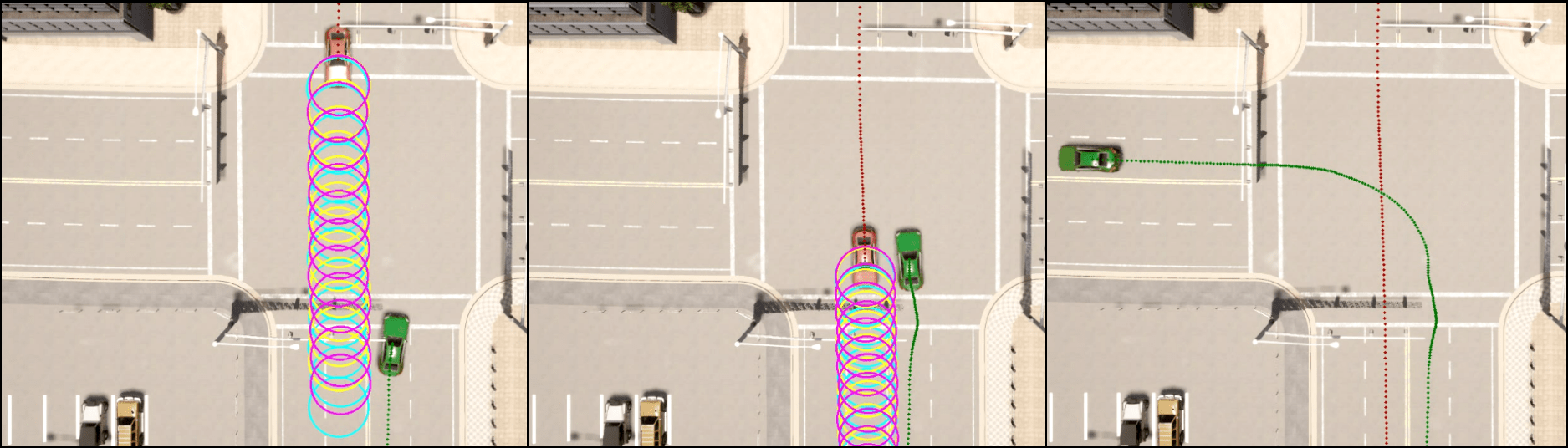}
  \caption{\small{SMPC from \cite{nair2021stochastic}}}
  \label{fig:hfsim1}
\end{subfigure}
\begin{subfigure}{1.05\columnwidth}
  \centering
  \includegraphics[width=1.0\columnwidth]{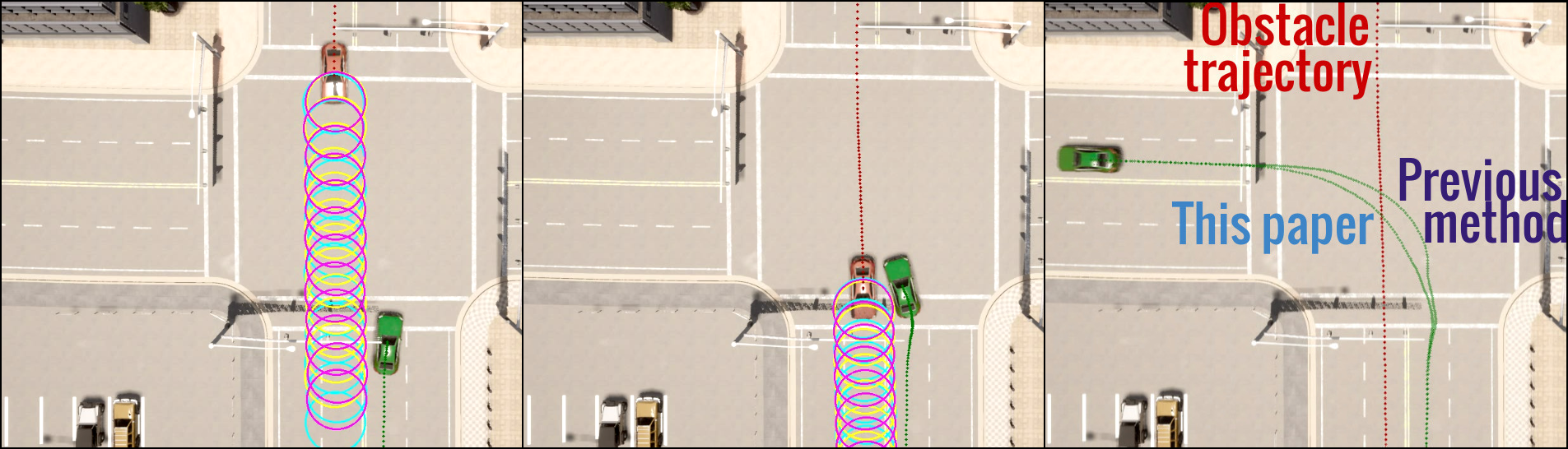}
  \caption{\small{SMPC from this paper}}
  \label{fig:hfsim2}
\end{subfigure}
\vskip -0.1 true in
    \caption{\small{Closed-loop trajectories of the agent (green) with the SMPC from this paper versus that of \cite{nair2021stochastic} for~\ref{ssec:carla}, Experiment 2. The coloured ellipses are the confidence sets given by the multi-modal predictions for the obstacle (red). The improved collision avoidance formulation allows for a tighter turn. Video: \url{https://youtu.be/wgqO36a1SU8}}}
    \label{fig:carla_sim}
    \vskip -0.3 true in
\end{figure}




\section{Conclusion}
\label{sec:conclusion}
We proposed convex MPC formulations for collision avoidance with dynamic obstacles and prediction uncertainties given by 1) Polytopic supports, 2) Gaussian distributions and 3) Arbitrary distributions with known mean and variance, and numerically validated our approach at traffic intersection scenarios. The key idea is to use dual perspective of collision avoidance, and tighten the reformulated dual feasibility problem  for different prediction uncertainties. This approach can be applied for agent and obstacle geometries given by general convex, cone-representable sets such as ellipsoids and polytopes.

\bibliographystyle{IEEEtran}

\bibliography{references.bib}

\begin{thebibliography}{10}
\providecommand{\url}[1]{#1}
\csname url@rmstyle\endcsname
\providecommand{\newblock}{\relax}
\providecommand{\bibinfo}[2]{#2}
\providecommand\BIBentrySTDinterwordspacing{\spaceskip=0pt\relax}
\providecommand\BIBentryALTinterwordstretchfactor{4}
\providecommand\BIBentryALTinterwordspacing{\spaceskip=\fontdimen2\font plus
\BIBentryALTinterwordstretchfactor\fontdimen3\font minus
  \fontdimen4\font\relax}
\providecommand\BIBforeignlanguage[2]{{%
\expandafter\ifx\csname l@#1\endcsname\relax
\typeout{** WARNING: IEEEtran.bst: No hyphenation pattern has been}%
\typeout{** loaded for the language `#1'. Using the pattern for}%
\typeout{** the default language instead.}%
\else
\language=\csname l@#1\endcsname
\fi
#2}}

\bibitem{nhtsa}
N.~H. T.~S. Administration, ``Automated vehicles for safety,''
  \emph{https://www.nhtsa.gov/technology-innovation/automated-vehicles-safety},
  2020.

\bibitem{multipath_2019}
Y.~Chai, B.~Sapp, M.~Bansal, and D.~Anguelov, ``Multipath: Multiple
  probabilistic anchor trajectory hypotheses for behavior prediction,''
  \emph{arXiv preprint arXiv:1910.05449}, 2019.

\bibitem{trajectron_2020}
T.~Salzmann, B.~Ivanovic, P.~Chakravarty, and M.~Pavone, ``Trajectron++:
  Dynamically-feasible trajectory forecasting with heterogeneous data,'' in
  \emph{2020 ECCV}.\hskip 1em plus 0.5em minus 0.4em\relax Springer, 2020.

\bibitem{morari1999model}
M.~Morari and J.~H. Lee, ``Model predictive control: past, present and
  future,'' \emph{Computers \& Chemical Engineering}, 1999.

\bibitem{benblog}
B.~Recht, ``What we've learned to control,''
  \emph{https://www.argmin.net/2020/06/29/tour-revisited/}, 2020.

\bibitem{shen2021collision}
X.~Shen, E.~L. Zhu, Y.~R. St{\"u}rz, and F.~Borrelli, ``Collision avoidance in
  tightly-constrained environments without coordination: a hierarchical control
  approach,'' in \emph{2021 IEEE International Conference on Robotics and
  Automation (ICRA)}.\hskip 1em plus 0.5em minus 0.4em\relax IEEE, 2021.

\bibitem{brudigam2021stochastic}
T.~Br{\"u}digam, M.~Olbrich, D.~Wollherr, and M.~Leibold, ``Stochastic model
  predictive control with a safety guarantee for automated driving,''
  \emph{IEEE Transactions on Intelligent Vehicles}, 2021.

\bibitem{zhou2018joint}
B.~Zhou, W.~Schwarting, D.~Rus, and J.~Alonso-Mora, ``Joint multi-policy
  behavior estimation and receding-horizon trajectory planning for automated
  urban driving,'' in \emph{2018 IEEE International Conference on Robotics and
  Automation (ICRA)}.\hskip 1em plus 0.5em minus 0.4em\relax IEEE, 2018.

\bibitem{wang2020non}
A.~Wang, A.~Jasour, and B.~C. Williams, ``Non-gaussian chance-constrained
  trajectory planning for autonomous vehicles under agent uncertainty,''
  \emph{IEEE Robotics and Automation Letters}, 2020.

\bibitem{de2021scenario}
O.~de~Groot, B.~Brito, L.~Ferranti, D.~Gavrila, and J.~Alonso-Mora,
  ``Scenario-based trajectory optimization in uncertain dynamic environments,''
  \emph{IEEE Robotics and Automation Letters}, 2021.

\bibitem{castillo2020real}
M.~Castillo-Lopez, P.~Ludivig, S.~A. Sajadi-Alamdari, J.~L. Sanchez-Lopez,
  M.~A. Olivares-Mendez, and H.~Voos, ``A real-time approach for
  chance-constrained motion planning with dynamic obstacles,'' \emph{IEEE
  Robotics and Automation Letters}, 2020.

\bibitem{bujarbaruah2021learning}
M.~Bujarbaruah, Y.~R. St{\"u}rz, C.~Holda, K.~H. Johansson, and F.~Borrelli,
  ``Learning environment constraints in collaborative robotics: A decentralized
  leader-follower approach,'' in \emph{2021 IEEE/RSJ International Conference
  on Intelligent Robots and Systems (IROS)}.\hskip 1em plus 0.5em minus
  0.4em\relax IEEE, 2021.

\bibitem{zhu2019chance}
H.~Zhu and J.~Alonso-Mora, ``Chance-constrained collision avoidance for mavs in
  dynamic environments,'' \emph{IEEE Robotics and Automation Letters}, 2019.

\bibitem{dixit2021risk}
A.~Dixit, M.~Ahmadi, and J.~W. Burdick, ``Risk-sensitive motion planning using
  entropic value-at-risk,'' in \emph{2021 European Control Conference
  (ECC)}.\hskip 1em plus 0.5em minus 0.4em\relax IEEE.

\bibitem{borrelli2017predictive}
F.~Borrelli, A.~Bemporad, and M.~Morari, \emph{Predictive control for linear
  and hybrid systems}.\hskip 1em plus 0.5em minus 0.4em\relax Cambridge
  University Press, 2017.

\bibitem{canny1988complexity}
J.~Canny, \emph{The complexity of robot motion planning}.\hskip 1em plus 0.5em
  minus 0.4em\relax MIT press, 1988.

\bibitem{cruise}
``Webwiz.''\hskip 1em plus 0.5em minus 0.4em\relax Cruise Automation,
  https://webviz.io/, 2019.

\bibitem{zhang2020optimization}
X.~Zhang, A.~Liniger, and F.~Borrelli, ``Optimization-based collision
  avoidance,'' \emph{IEEE Transactions on Control Systems Technology}, 2020.

\bibitem{soloperto2019collision}
R.~Soloperto, J.~K{\"o}hler, F.~Allg{\"o}wer, and M.~A. M{\"u}ller, ``Collision
  avoidance for uncertain nonlinear systems with moving obstacles using robust
  model predictive control,'' in \emph{2019 18th European Control Conference
  (ECC)}.\hskip 1em plus 0.5em minus 0.4em\relax IEEE, 2019.

\bibitem{goulart2006optimization}
P.~J. Goulart, E.~C. Kerrigan, and J.~M. Maciejowski, ``Optimization over state
  feedback policies for robust control with constraints,'' \emph{Automatica},
  2006.

\bibitem{balci2021covariance}
I.~M. Balci and E.~Bakolas, ``Covariance control of discrete-time gaussian
  linear systems using affine disturbance feedback control policies,''
  \emph{arXiv preprint arXiv:2103.14428}, 2021.

\bibitem{rahimian2019distributionally}
H.~Rahimian and S.~Mehrotra, ``Distributionally robust optimization: A
  review,'' \emph{arXiv preprint arXiv:1908.05659}, 2019.

\bibitem{mccormick1976computability}
G.~P. McCormick, ``Computability of global solutions to factorable nonconvex
  programs: Part i—convex underestimating problems,'' \emph{Mathematical
  programming}, 1976.

\bibitem{nair2021stochastic}
S.~H. Nair, V.~Govindarajan, T.~Lin, C.~Meissen, H.~E. Tseng, and F.~Borrelli,
  ``Stochastic mpc with multi-modal predictions for traffic intersections,''
  \emph{arXiv preprint arXiv:2109.09792}, 2021.

\bibitem{andersson2019casadi}
J.~A. Andersson, J.~Gillis, G.~Horn, J.~B. Rawlings, and M.~Diehl, ``Casadi: a
  software framework for nonlinear optimization and optimal control,''
  \emph{Mathematical Programming Computation}, 2019.

\bibitem{gurobi}
\BIBentryALTinterwordspacing
{Gurobi Optimization, LLC}, ``{Gurobi Optimizer Reference Manual},'' 2021.
  [Online]. Available: \url{https://www.gurobi.com}
\BIBentrySTDinterwordspacing

\bibitem{carla_sim_2017}
A.~Dosovitskiy, G.~Ros, F.~Codevilla, A.~Lopez, and V.~Koltun, ``{CARLA}: {An}
  open urban driving simulator,'' in \emph{Proceedings of the 1st Annual
  Conference on Robot Learning}, 2017.

\bibitem{boyd2004convex}
S.~Boyd, S.~P. Boyd, and L.~Vandenberghe, \emph{Convex optimization}.\hskip 1em
  plus 0.5em minus 0.4em\relax Cambridge university press, 2004.

\bibitem{calafiore2006distributionally}
G.~C. Calafiore and L.~E. Ghaoui, ``On distributionally robust
  chance-constrained linear programs,'' \emph{Journal of Optimization Theory
  and Applications}, 2006.

\end{thebibliography}
\appendix
\subsection{Matrix definitions}\label{app:matrices}
\scriptsize
\begin{align}
    &\mathbf{h}_t=[h^\top_{t|t}\dots h^{\top}_{t+N-1|t}]^\top,\mathbf{K}_t=\text{blkdiag}\left(K_{t|t},\dots, K_{t+N-1|t}\right),\label{mat:hK}\\
    &\mathbf{M}_t=\begin{bmatrix}
    O&\hdots&\hdots& O\\
    M_{t,t+1|t}&O&\hdots& O\\
    \vdots&\vdots&\vdots&\ddots\\
    M_{t,t+N-1|t}&\hdots & M_{t+N-2,t+N-1|t}&O
    \end{bmatrix},\label{mat:M}\\
    &\mathbf{A}_t=\begin{bmatrix}I_{n_x}\\ A_{t}\\ \vdots\\\prod\limits_{k=t}^{t+N-1}A_{k}\end{bmatrix},  \mathbf{B}_t=\begin{bmatrix}O&\hdots&\hdots& O\\B_{t}&O&\hdots&O\\\vdots&\ddots&\ddots&\vdots\\\prod\limits_{k=t+1}^{t+N-1}A_{k}B_{t}&\hdots&\dots&B_{t+N-1}\end{bmatrix},\label{mat:AB}\\
    &\mathbf{T}_t=\begin{bmatrix}I_{n_x}\\ T_{t}\\\vdots\\\prod\limits_{k=t}^{t+N-1}T_{k}\end{bmatrix}, \mathbf{q}_t=\begin{bmatrix}O\\q_{t}\\\vdots\\q_{t+N-1}+\sum\limits_{k=t}^{t+N-1}\prod\limits_{l=k+1}^{t+N-1}T_{l} q_{k}\end{bmatrix},\label{mat:Tq}\\
    &\mathbf{E}_{t}=\begin{bmatrix}O&\hdots& O\\E_t&\hdots&O\\\vdots&\ddots&\vdots\\\prod\limits_{k=t+1}^{t+N-1}A_{k}E_t&\dots&E_{t+N-1}\end{bmatrix}, \mathbf{F}_{t}=\begin{bmatrix}O&\hdots&\hdots& O\\F_t&O&\hdots&O\\\vdots&\ddots&\ddots&\vdots\\\prod\limits_{k=t+1}^{t+N-1}T_{k}F_t&\hdots&\dots&F_{t+N-1}\end{bmatrix}\label{mat:ER}
    \end{align}
\normalsize 
\subsection{Proof of Proposition~\ref{prop:oa_r}}\label{app:prop1proof}
Consider the Lagrangian of the optimization problem \eqref{eq:dist}, {\small{$\mathcal{L}(z_1,z_2,\lambda^i_{k|t}, \nu^i_{k|t})=||z_1-z_2||_2+\lambda^{i\top}_{k|t}(GR^\top_k(z_1-p_{k|t})-g)+\nu^{i\top}_{k|t}(G^iR^{i\top}_k(z_2-p^i_{k|t})-g^i)$}} where the Lagrange multipliers $\lambda^i_{k|t},\nu^i_{k|t}\in\mathcal{K}^*$. Define the dual objective as
\small
\begin{align*}
    d(\lambda^i_{k|t}, \nu^i_{k|t})&=\text{inf}_{z_1,z_2}\mathcal{L}(z_1,z_2,\lambda^i_{k|t}, \nu^i_{k|t})\\
    &=-\lambda^{i\top}_{k|t}(GR^\top_kp_{k|t}+g)-\nu^{i\top}_{k|t}(GR^{i\top}_kp^i_{k|t}+g^i)\\&~-\sup_{z_1,z_2}\Big(\begin{bmatrix}-R_kG^\top\lambda^i_{k|t}\\-R^i_kG^{i\top}\nu^i_{k|t}\end{bmatrix}^\top\begin{bmatrix}z_1\\z_2\end{bmatrix}-f(z_1-z_2)\Big)
\end{align*}
\normalsize
where $f(\cdot)=||\cdot||_2$. We use properties of convex conjugates\cite{boyd2004convex}  to obtain the dual objective. The convex conjugate of $f(\cdot)$ is given by {\small{$f^*(y)=\sup_x y^\top x-f(x)=\{ 0\text{ if }|| y ||_2\leq 1,\ \infty\text{ otherwise }\}$}}. Moreover if {\small{$h(x)=f(Ax)$}}, then {\small{$h^*(y)=\inf_{A^\top z=y}f^*(z)$}}. For {\small{$y=-[\lambda^{i\top}_{k|t}GR^\top_k\ \nu^{i\top}_{k|t}G^{i}R^{i\top}_k]^\top$}}, {\small{$A=[I_n\ -I_n]$}}, {\small{$x=[z_1^\top z_2^\top]^\top$}}, the dual function can now be written as follows,
\small
\begin{align*}
    d(\lambda^i_{k|t}, \nu^i_{k|t})
    &=-\lambda^{i\top}_{k|t}(GR^\top_kp_{k|t}+g)-\nu^{i\top}_{k|t}(GR^{i\top}_kp^i_{k|t}+g^i)-h^*(y)\\
    &=-\lambda^{i\top}_{k|t}(GR^\top_kp_{k|t}+g)-\nu^{i\top}_{k|t}(GR^{i\top}_kp^i_{k|t}+g^i)
\end{align*}
\normalsize
for {\small{$\Vert\lambda^{i\top}_{k|t}GR_k^\top\Vert\leq 1, \Vert\nu^{i\top}_{k|t}G^iR^{i\top}_k\Vert_2 \leq 1, R_kG^\top\lambda^i_{k|t}+R^i_{k|t}G^{i\top}_{k|t}\nu^i_{k|t}=0$}}. The dual optimization problem for \eqref{eq:dist} can now be written as
\small
\begin{align}\label{eq:dual_dist}
    \max_{\substack{\lambda^i_{k|t},\nu^i_{k|t}\in\mathcal{K}^*}}&-\lambda^{i\top}_{k|t}(GR^\top_kp_{k|t}+g)-\nu^{i\top}_{k|t}(GR^{i\top}_kp^i_{k|t}+g^i)\nonumber\\
    \text{ s.t}&~ \Vert\lambda^{i\top}_{k|t}GR_k^\top\Vert_2\leq1,\Vert \nu^{i\top}_{k|t}G^iR^{i\top}_k\Vert_2 \leq1,\nonumber\\
    &~~R_kG^\top\lambda^i_{k|t}+R^i_{k|t}G^{i\top}_{k|t}\nu^i_{k|t}=0.
\end{align}
\normalsize

Since the sets 
$\{e| Ge\preceq_{\mathcal{K}}g\}, \{o| G^io\preceq_{\mathcal{K}}g^i\}$ are non-empty, the feasible set of \eqref{eq:dist} has a non-empty interior. By Slater's condition, strong duality holds and thus $\text{dist}(\mathbb{S}_k(x_{k|t}), \mathbb{S}^i_k(o^i_{k|t}))$ is equal the optimal objective of \eqref{eq:dual_dist}. Thus, we rewrite $\text{dist}(\mathbb{S}_k(x_{k|t}), \mathbb{S}^i_k(o^i_{k|t}))>0$ as
\small
\begin{align*}
 &\max_{\substack{\lambda^i_{k|t},\nu^i_{k|t}\in\mathcal{K}^*,\Vert\lambda^{i\top}_{k|t}GR_k^\top\Vert_2\leq 1\\ R_kG^\top\lambda^i_{k|t}=-R^i_{k|t}G^{i\top}_{k|t}\nu^i_{k|t}}}d(\lambda^i_{k|t}, \nu^i_{k|t})>0\\
    \Leftrightarrow&  \exists\lambda^i_{k|t},\nu^i_{k|t}\in\mathcal{K}^*:-\lambda^{i\top}_{k|t}(GR^\top_k(p_{k|t}-p^i_{k|t})+g)-\nu^{i\top}_{k|t}g^i>0,\\ &~\Vert\lambda^{i\top}_{k|t}GR_k^\top\Vert_2\leq 1, R_kG^\top\lambda^i_{k|t}=-R^i_{k|t}G^{i\top}_{k|t}\nu^i_{k|t}
\end{align*}
\normalsize$\hfill\blacksquare$
\subsection{Proof of Theorem~\ref{thm:constr_det_r}}\label{app:thm1proof}
  For all $k\in\mathcal{I}_t^{t+N-1}$, $i\in\mathcal{I}_1^M$, define the sets 
 {\footnotesize{
 \begin{align*}
    &\mathcal{S}^i_{k|t}(\mathbf{w}_t,\mathbf{n}_t)=\left\{
    \begin{bmatrix}\mathbf{x}_t\\\mathbf{u}_t\\\mathbf{o}_t\\
    \boldsymbol{\lambda}_{t}\\\boldsymbol{\nu}_{t}\end{bmatrix}\middle\vert
    \begin{aligned}
    &\lambda^i_{k|t},\nu^i_{k|t}\in\mathcal{K}^*,\Vert\lambda^{i\top}_{k|t}GR_{k}^\top\Vert_2\leq 1,\\
    &\lambda^{i\top}_{k|t}GR_{k}^\top=- \nu^{i\top}_{k|t}G^iR^{i\top}_{k},\\
    &\lambda^{i\top}_{k|t}(GR_{k}^\top(C(S^x_{k}\mathbf{x}_t-S^{o,i}_k\mathbf{o}_t)+c_t)+g)\\
    &<-\nu^{i\top}_{k|t}g_i,\\
    &\mathbf{x}_t=\mathbf{A}_tx_t+\mathbf{B}_t\mathbf{u}_t+\mathbf{E}_t\mathbf{w}_t,\\ &\mathbf{o}_t=\mathbf{T}_to_t+\mathbf{q}_t+\mathbf{F}_t\mathbf{n}_t
    \end{aligned}
    \right\}
 \end{align*}
 \begin{align*}
    &\bar{\mathcal{S}}_{k|t}(\mathbf{w}_t,\mathbf{n}_t)=\left\{
    (\mathbf{x}_t,\mathbf{u}_t)\middle\vert
    \begin{aligned}
    &F_j^xS^x_{k+1}\mathbf{x}_t+F_j^uS^u_k\mathbf{u}_t\leq f_j,~\forall j\in\mathcal{I}_1^J\\
    &\mathbf{x}_t=\mathbf{A}_tx_t+\mathbf{B}_t\mathbf{u}_t+\mathbf{E}_t\mathbf{w}_t
    \end{aligned}
    \right\}
 \end{align*}}}  
 and see that \small\begin{align}\label{eq:F_split}
     &\mathcal{S}_t(\mathbf{w}_t,\mathbf{n}_t)=\bigcap_{k=t}^{t+N-1}\left(\bar{\mathcal{S}}_{k|t}(\mathbf{w}_t,\mathbf{n}_t)\bigcap_{i=1}^M\mathcal{S}^i_{k+1|t}(\mathbf{w}_t,\mathbf{n}_t)\right)
     \end{align}
     \normalsize
 We proceed to derive the feasible set for each case as follows.
 
 \subsubsection*{1)} We use \eqref{eq:F_split} to express \eqref{constr:robust} as
 \small
 \begin{align*}
 \mathcal{C}(\mathbf{D1})=\bigcap_{k=t}^{t+N-1}\left(\bigcap_{\mathbf{v}_t\in\mathcal{D}^N}\bar{\mathcal{S}}_{k|t}(\mathbf{w}_t,\mathbf{n}_t)\bigcap_{i=1}^M\mathcal{S}^i_{k+1|t}(\mathbf{w}_t,\mathbf{n}_t)\right)
 \end{align*}
 \normalsize
 For any $k\in\mathcal{I}_{t+1}^{t+N}$, $i\in\mathcal{I}_1^M$, a feasible point in $\bigcap_{\mathbf{v}_t\in\mathcal{D}^N}\mathcal{S}^i_{k|t}(\mathbf{w}_t,\mathbf{n}_t)$ satisfies the constraints
 \small
 \begin{align*}
    &-\lambda^{i\top}_{k|t}(GR_{k}^\top(C(S^x_{k}\mathbf{x}_t-S^{o,i}_k\mathbf{o}_t)+c_t)+g)>\nu^{i\top}_{k|t}g_i, \forall\mathbf{v}_t\in\mathcal{D}^N\\
     &\lambda^i_{k|t},\nu^i_{k|t}\in\mathcal{K}^*,\Vert\lambda^{i\top}_{k|t}GR_k^\top\Vert_2\leq 1,\lambda^{i\top}_{k|t}GR_k^\top=- \nu^{i\top}_{k|t}G^iR^{i\top}_k.
      \end{align*}
      \normalsize
      Plugging in the closed-loop evolution of $\mathbf{x}_t,\mathbf{o}_t$ using $ \mathbf{u}_t=\mathbf{h}_t+\mathbf{M}_t\mathbf{w}_t+\mathbf{K}_t\mathbf{F}_t\mathbf{n}_t$, we can rewrite this using the functions $Y^i_k(\boldsymbol{\theta_t},\lambda^i_{k|t},\nu^i_{k|t})$, $Z^i_k(\boldsymbol{\theta}_t,\lambda^i_{k|t},\nu^i_{k|t})$ as
      \small
      \begin{align*}
  & Y^i_k(\boldsymbol{\theta_t},\lambda^i_{k|t},\nu^i_{k|t})+Z^i_k(\boldsymbol{\theta_t},\lambda^i_{k|t},\nu^i_{k|t})\begin{bmatrix}\mathbf{w}_t\\\mathbf{n}_t\end{bmatrix}>0,~\forall\mathbf{v}_t\in\mathcal{D}^N,\\
     &\lambda^i_{k|t},\nu^i_{k|t}\in\mathcal{K}^*,\Vert\lambda^{i\top}_{k|t}GR_k^\top\Vert_2\leq 1,\lambda^{i\top}_{k|t}GR_k^\top=- \nu^{i\top}_{k|t}G^iR^{i\top}_k.
 \end{align*}
 \normalsize
   The first inequality in the last implication can be equivalently expressed without the quantifier $\forall\mathbf{v}_t\in\mathcal{D}^N$ as
 \small
 \begin{align*}
 & Y^i_k(\boldsymbol{\theta}_t,\lambda^i_{k|t},\nu^i_{k|t})+\min_{\mathbf{v}_t\in\mathcal{D}^N}Z^i_k(\boldsymbol{\theta}_t,\lambda^i_{k|t},\nu^i_{k|t})\mathbf{P}\mathbf{v}_t>0\\
 \Leftrightarrow & Y^i_k(\boldsymbol{\theta}_t,\lambda^i_{k|t},\nu^i_{k|t})-\max_{\Vert \tilde{\mathbf{d}}\Vert_\infty \leq 1}-Z^i_k(\boldsymbol{\theta}_t,\lambda^i_{k|t},\nu^i_{k|t})\mathbf{P}\gamma\mathbf{\Gamma}^{-1}\tilde{\mathbf{d}}>0\\
 \Leftrightarrow & Y^i_k(\boldsymbol{\theta}_t,\lambda^i_{k|t},\nu^i_{k|t})-\gamma\Vert Z^i_k(\boldsymbol{\theta}_t,\lambda^i_{k|t},\nu^i_{k|t})\mathbf{P}\mathbf{\Gamma}^{-1}\Vert_1>0
 \end{align*}
 \normalsize
 where the last implication is obtained by noting that $\Vert\cdot\Vert_1$ is the dual norm of $\Vert\cdot\Vert_\infty$.
 For the state-input constraints, we similarly have for each $j\in\mathcal{I}_1^J$, 
 \small
 \begin{align*}
     &F_j^xS^x_{k+1}\mathbf{x}_t+F_j^uS^u_k\mathbf{u}_t\leq f_j,~ \forall\mathbf{v}_t\in\mathcal{D}^N\\
     \Leftrightarrow&\bar{Y}^j_k(\boldsymbol{\theta_t})+\bar{Z}^j_k(\boldsymbol{\theta_t})\begin{bmatrix}\mathbf{w}_t\\\mathbf{n}_t\end{bmatrix}>0,~\forall\mathbf{v}_t\in\mathcal{D}^N, \\
     \Leftrightarrow& \bar{Y}^j_k(\boldsymbol{\theta}_t)-\gamma\Vert \bar{Z}^j_k(\boldsymbol{\theta}_t)\mathbf{P}\mathbf{\Gamma}^{-1}\Vert_1\geq0
 \end{align*}
 \normalsize
 The feasible set {\small{$\mathcal{F}_t(\mathbf{D1})$}} is thus defined by the above inequalities $\forall k\in\mathcal{I}_{t}^{t+N-1},\forall i\in\mathcal{I}_1^M, \forall j\in\mathcal{I}_1^J$.
 
 \subsubsection*{2)} Now we proceed to the chance-constrained case. For convenience, define $\mathbf{z}_t=[\mathbf{x}^\top_t\ \mathbf{u}^\top_t\ \mathbf{o}^\top_t\ \boldsymbol{\lambda}^\top_t\ \boldsymbol{\nu}_t]^\top$. The joint constraint {\small{$\mathbb{P}(\mathbf{z}_t\in\mathcal{S}_t(\mathbf{w}_t,\mathbf{n}_t))\geq1-\epsilon$}} is difficult to reformulate in the given form, so we construct an inner-approximation to this set using individual chance constraints as follows. \small\begin{align*}
  \mathbb{P}(&\mathbf{z}_t\in\mathcal{S}_t(\mathbf{w}_t,\mathbf{n}_t))\\
  &=\mathbb{P}(\mathbf{z}_t\in\bigcap_{k=t}^{t+N-1}\big(\bar{\mathcal{S}}_{k|t}(\mathbf{w}_t,\mathbf{n}_t)\bigcap_{i=1}^M\mathcal{S}^i_{k+1|t}(\mathbf{w}_t,\mathbf{n}_t)\big))\\
     \Rightarrow\mathbb{P}(&\mathbf{z}_t\not\in\mathcal{S}_t(\mathbf{w}_t,\mathbf{n}_t))\\
     &=\mathbb{P}(\mathbf{z}_t\not\in\bigcup_{k=t}^{t+N-1}\big(\bar{\mathcal{S}}_{k|t}(\mathbf{w}_t,\mathbf{n}_t)\bigcup_{i=1}^M\mathcal{S}^i_{k+1|t}(\mathbf{w}_t,\mathbf{n}_t)\big))\\
   &\leq\sum\limits_{k=t}^{t+N-1}\Big(\mathbb{P}(\mathbf{z}_t\not\in\bar{\mathcal{S}}_{k|t}(\mathbf{w}_t,\mathbf{n}_t))+\sum\limits_{i=1}^M\mathbb{P}(\mathbf{z}_t\not\in\mathcal{S}^i_{k+1|t}(\mathbf{w}_t,\mathbf{n}_t))\Big)
      \end{align*}
     \normalsize
 Thus, if each {\small{$\mathbb{P}(\mathbf{z}_t\not\in\bar{\mathcal{S}}_{k|t}(\mathbf{w}_t,\mathbf{n}_t))\leq\frac{\epsilon}{2N}$}} and {\small{ $\mathbb{P}(\mathbf{z}_t\not\in\mathcal{S}^i_{k+1|t}(\mathbf{w}_t,\mathbf{n}_t))\leq\frac{\epsilon}{2NM}$}}, we get {\small{ $\mathbb{P}(\mathbf{z}_t\in\mathcal{S}_t(\mathbf{w}_t,\mathbf{n}_t))=1-\mathbb{P}(\mathbf{z}_t\not\in\mathcal{S}_t(\mathbf{w}_t,\mathbf{n}_t))\geq 1-\epsilon$}}. 
 So for any $k\in\mathcal{I}_{t+1}^{t+N-1}$, we reformulate {\small{$\mathbb{P}(\mathbf{z}_t\in\mathcal{S}^i_{k|t}(\mathbf{w}_t,\mathbf{n}_t))\geq 1-\frac{\epsilon}{2NM}$}} using the result from \cite{calafiore2006distributionally}: \small$\mathbb{P}(a^\top x >b)\geq 1-\epsilon, x\sim\mathcal{N}(\mu_x,\Sigma_x)\Leftrightarrow a^\top\mu_x-\Phi^{-1}(1-\epsilon)\Vert a^\top\Sigma^{\frac{1}{2}}_x\Vert_2>b$\normalsize. 
 \small
 \begin{align*}
   &\mathbb{P}(\mathbf{z}_t\in\mathcal{S}^i_{k|t}(\mathbf{w}_t,\mathbf{n}_t))\geq 1-\frac{\epsilon}{2NM},\\
     \Leftrightarrow&\mathbb{P}( Y^i_k(\boldsymbol{\theta}_t,\lambda^i_{k|t},\nu^i_{k|t})+Z^i_k(\boldsymbol{\theta}_t,\lambda^i_{k|t},\nu^i_{k|t})\begin{bmatrix}\mathbf{w}_t\\\mathbf{n}_t\end{bmatrix}>0)\geq 1-\epsilon,\\
     &\lambda^i_{k|t},\nu^i_{k|t}\in\mathcal{K}^*,\Vert\lambda^{i\top}_{k|t}GR_k^\top\Vert_2\leq 1,\lambda^{i\top}_{k|t}GR_k^\top=- \nu^{i\top}_{k|t}G^iR^{i\top}_k,\\
     \Leftrightarrow&Y^i_k(\boldsymbol{\theta}_t,\lambda^i_{k|t},\nu^i_{k|t})-\gamma_{ca}\Vert Z^i_k(\boldsymbol{\theta}_t,\lambda^i_{k|t},\nu^i_{k|t})\mathbf{P}\mathbf{\Sigma}^{\frac{1}{2}}\Vert_2>0,\\
     &\lambda^i_{k|t},\nu^i_{k|t}\in\mathcal{K}^*,\Vert\lambda^{i\top}_{k|t}GR_k^\top\Vert_2\leq 1,\lambda^{i\top}_{k|t}GR_k^\top=- \nu^{i\top}_{k|t}G^iR^{i\top}_k.
 \end{align*}
 \normalsize
 where {\small{$\gamma_{ca}=\Phi^{-1}(1-\frac{\epsilon}{2NM})$}}.
The joint-chance constraint {\small{$\mathbb{P}((\mathbf{x}_t,\mathbf{u}_t)\not\in\bar{\mathcal{S}}_{k|t}(\mathbf{w}_t,\mathbf{n}_t))\leq\frac{\epsilon}{2N}$}} is similarly inner-approximated by imposing {\small{$\mathbb{P}(F_j^xS^x_{k+1}\mathbf{x}_t+F_j^uS^u_k\mathbf{u}_t> f_j)\leq\frac{\epsilon}{2NJ}~\forall j\in\mathcal{I}_1^J$}}, which is given by
\small
\begin{align*}
&\bar{Y}^j_k(\boldsymbol{\theta}_t)-\gamma_{xu}\Vert \bar{Z}^j_k(\boldsymbol{\theta}_t)\mathbf{P}\mathbf{\Sigma}^{\frac{1}{2}}\Vert_2\geq0,~\forall j\in\mathcal{I}_1^J.
\end{align*}
\normalsize
where $\gamma_{xu}=\Phi^{-1}(1-\frac{\epsilon}{2NJ})$.
Thus, the inner-approximation of the feasible set is given by {\small{$\mathcal{F}_t(\mathbf{D2})$}}, defined by the above inequalities $\forall k\in\mathcal{I}_{t}^{t+N-1},\forall i\in\mathcal{I}_1^M, \forall j\in\mathcal{I}_1^J$.
 \subsubsection*{3)} For the distributionally robust case, the joint-chance constraints are converted to individual chance constraints in the same way as shown above. However, since only the first two moments of the probability distribution are known, we reformulate the constraint using the result from \cite{calafiore2006distributionally}: {\small{$\inf_{\substack{P\in\bar{\mathcal{P}}\\x\sim P}}\mathbb{P}(a^\top x>b)\geq 1-\epsilon\Leftrightarrow a^\top\mu_x-\sqrt{\frac{1-\epsilon}{\epsilon}}\Vert a^\top\Sigma^{\frac{1}{2}}_x\Vert_2>b$}} , where $\bar{\mathcal{P}}$ is the set of distributions with mean $\mu_x$ and variance $\Sigma_x$. Proceeding identically as in the chance-constrained case but with {\small{$\gamma_{ca}=\sqrt{\frac{2NM-\epsilon}{\epsilon}}$}} and {\small{$\gamma_{xu}=\sqrt{\frac{2NJ-\epsilon}{\epsilon}}$}}, we get {\small{$\mathcal{F}_t(\mathbf{D3})$}}.
 
 $\hfill\blacksquare$

\end{document}